\documentclass{article} %
\usepackage{arxiv_formatting,times}

\usepackage{amsmath,amsfonts,bm}

\def\eqref#1{equation~\ref{#1}}
\def\1{\bm{1}}

\DeclareMathAlphabet{\mathsfit}{\encodingdefault}{\sfdefault}{m}{sl}
\SetMathAlphabet{\mathsfit}{bold}{\encodingdefault}{\sfdefault}{bx}{n}

\usepackage{xcolor}
\usepackage{float}
\usepackage{hyperref}
\definecolor{citecolor}{HTML}{2A78D6}
\definecolor{linkcolor}{HTML}{4A3AA7}
\definecolor{urlcolor}{HTML}{0E9E8A}
\hypersetup{colorlinks=true, citecolor=citecolor, linkcolor=linkcolor, urlcolor=urlcolor}
\usepackage{url}

\iclrfinalcopy
\usepackage{graphicx}
\usepackage{subcaption}
\usepackage{cleveref}
\usepackage{xurl}
\usepackage{booktabs}
\usepackage{wrapfig}
\usepackage{xcolor}
\usepackage[bottom]{footmisc}
\usepackage{amsthm}
\usepackage{amsmath}
\usepackage{amssymb}
\usepackage{mathtools}
\usepackage{titletoc}

\definecolor{rlColor}{HTML}{F28C28}
\definecolor{distilledColor}{HTML}{2A78D6}
\definecolor{baseTextColor}{HTML}{696969}

\usepackage[most]{tcolorbox}

\definecolor{tracered}{RGB}{240,180,180}
\definecolor{tracegreen}{RGB}{180,220,180}
\definecolor{codebg}{RGB}{245,245,245}
\definecolor{traceorange}{HTML}{F0D2B4} 
\definecolor{tracegrey}{RGB}{210,210,210}
\definecolor{traceyellow}{RGB}{245,220,155}

\lstdefinestyle{tracecode}{
  basicstyle=\ttfamily\footnotesize,
  backgroundcolor=\color{codebg},
  breaklines=true,
  columns=fullflexible,
  keepspaces=true,
  frame=none,
  xleftmargin=4pt,
  xrightmargin=4pt,
  aboveskip=4pt,
  belowskip=4pt,
}

\newtcolorbox{tracebox}[1]{
  colback=#1,
  colframe=#1,
  arc=4mm,
  boxrule=0pt,
  left=6pt,right=6pt,top=6pt,bottom=6pt,
}
\newcommand{\bs}{\textbackslash}
\newcommand{\op}{\textasciigrave\textasciigrave\textasciigrave}

\theoremstyle{plain}

\newtheorem*{theorem*}{Theorem}

\theoremstyle{definition}

\theoremstyle{remark}

\fancypagestyle{firstpage}{
  \fancyhead{} %
  \renewcommand{\headrulewidth}{0pt} %
}

\title{Distillation Defenses Easily Break After\\Reinforcement Learning}

\author{%
\makebox[\linewidth][c]{\normalfont
\rule{0pt}{35pt}%
\begin{tabular}{@{}c@{}}
\textbf{Shidan Javaheri}$^{1}$ \quad
\textbf{Alexander Panfilov}$^{2}$ \quad
\textbf{Oliver Britton}$^{1}$ \\
\textbf{Yarin Gal}$^{1}$ \quad
\textbf{Yonatan Gideoni}$^{1}$ \\[0.3em]
$^{1}$University of Oxford \quad $^{2}$ELLIS Institute Tübingen \& MPI for Intelligent Systems\quad \\[0.3em]
\texttt{shidan@javaheri.org} \quad
\texttt{yg@robots.ox.ac.uk}
\end{tabular}}
\vspace{-15pt}
}

\begin{document}

\maketitle

\begin{abstract}
Distillation attacks copy the reasoning capabilities of closed-source large language models, allowing bad actors to replicate state-of-the-art performance at low cost. Attackers systematically collect a large volume of frontier model reasoning traces and then train (i.e., ``distill'') their own models on these traces. Existing defenses against distillation attacks are typically evaluated immediately after distillation, implicitly assuming attackers do not train their models any further. In this paper, we argue that a more realistic threat model includes further training with reinforcement learning after distillation. A misspecified threat model can give a false sense of security -- some defenses that seem effective after distillation can be broken after subsequent reinforcement learning. Practically, reinforcement learning lowers the bar for a distillation attack to be effective. We show that simple attacks can steal reasoning capabilities from existing closed-source language models using data easily obtainable from current APIs, yielding reasoning improvements equivalent to more sophisticated attacks that extract the full hidden traces. Results indicate that any distillation defense that leaks sufficient information to reconstruct approximate reasoning traces is likely ineffective. We conclude by discussing broader implications and batch-level distillation defenses which could be more effective.
\end{abstract}

\vspace{-3pt}
\section{Introduction}
\vspace{-3pt}
In early 2025, DeepSeek-R1 demonstrated that a Large Language Model (LLM) trained outside of Google, Anthropic, and OpenAI could outperform those trained by frontier labs for the first time \citep{deepseek}. This led to allegations accusing DeepSeek of using outputs obtained from these labs' models to improve its performance \citep{nytimes_openai_deepseek_2025,guardian_openai_deepseek_2025}, allegations that were later rebutted \citep{deepseek-nature, gibney2025secrets} and then resurfaced again \citep{anthropic, google, Seetharaman2026DeepSeekDistillation}. DeepSeek was accused of executing a distillation attack -- cheaply copying a closed-source model's reasoning capabilities by systematically collecting large quantities of high-quality reasoning traces and then training (``distilling'') its own LLMs on these traces (see  \Cref{fig:illustrate_distillation_attacks}).

\begin{figure}[!ht]
\vspace{-5pt}
    \centering
    \includegraphics[width=1\textwidth]{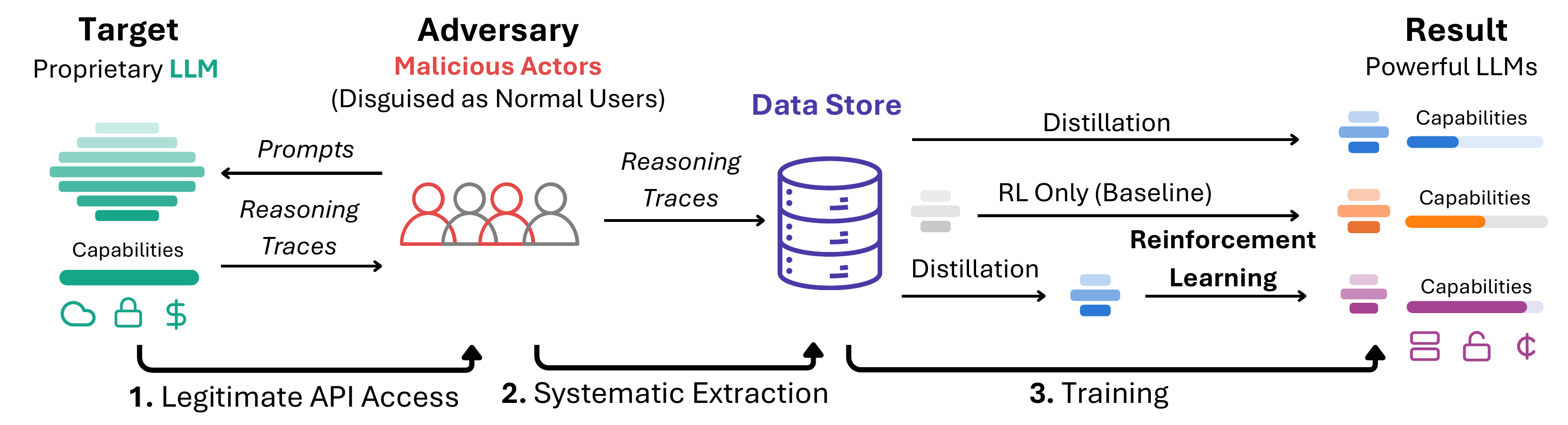}
    \caption{\textbf{A distillation attack's pipeline.} Attackers amass large volumes of reasoning traces from proprietary LLMs and train (``distill'') their own models on these traces, followed by further training of distilled models with reinforcement learning (RL). Existing threat models assume attackers train models only with distillation, while a realistic distillation attack likely follows distillation with RL. 
    }
    \label{fig:illustrate_distillation_attacks}
    \vspace{-5pt}
\end{figure}

Distillation attacks let attackers quickly and cheaply curate high-quality LLM training data, and can allow malicious actors to access powerful AI models without built-in safety guardrails \citep{google,antidist_news}. Guardrail-free models are especially dangerous with the rising capabilities of agentic AI, as demonstrated in recent accidental cybersecurity attacks \citep{openai2026huggingface,anthropic2026investigating}. 

Language models are trained to solve difficult, verifiable problems using reinforcement learning (RL). RL is the most compute-intensive part of the LLM post-training pipeline \citep{llama, qwen, deepseek}, and works by having a model generate many responses to difficult questions and then upweighting correct answers while downweighting incorrect answers \citep{grpo}. Although the training pipelines used in realistic distillation attacks are unknown, existing threat models implicitly assume that attackers do not perform any RL following a distillation attack, or that defenses effective after distillation remain effective after RL. Accordingly, previous work studying distillation attacks evaluates defenses by measuring an attacker's performance immediately post-distillation \citetext{\citealp{antidist_sampling,antidis_fingerprinting,zhang2026stealreasoningreasoningtraces,doge_defense,libon2026what,allouah2026distillation};\citealp[Sec.~5.1.1]{anthropic2026riskreport}}.

In this paper, we argue that existing threat models of distillation attacks are misspecified, resulting in existing distillation defenses giving a false sense of security. Specifically, we provide evidence that a realistic attack pipeline likely includes reinforcement learning training after distillation. Defenses that seem effective in post-distillation evaluations may be rendered ineffective after additional reinforcement learning, including those currently widely deployed in production systems.

We find that reinforcement learning lowers the bar for a distillation attack to be effective. Very simple attacks that reconstruct approximate reasoning traces break existing deployed defenses, namely those used by GPT, Claude, and Gemini models -- returning reasoning trace summaries instead of full traces. After RL training, summary-reconstructed reasoning traces yield performance improvements similar to distilling over full, unsummarized traces. These results imply that any distillation defense that leaks sufficient information to reconstruct reasoning traces will likely be ineffective. Reasoning traces are harder to reconstruct by omitting information from generated answers, but omissions would damage user experience due to dual-use; for example, a researcher can ask a model to generate a proof as part of their research or an attacker can ask the same to improve their own model's capabilities. Removing steps of the proof would harm both the researcher and the attacker alike.

\looseness=-1  This paper's contributions are presented as follows. In \Cref{sec:rl_in_threat_model} we show that models trained with reinforcement learning tend to outperform the same model trained with distillation. Furthermore, models trained with both distillation and reinforcement learning perform best. Realistic distillation attacks carried out by capable attackers are thus likely to train models with a combination of distillation and reinforcement learning. In \Cref{sec:false_security} we show that a well-known defense -- antidistillation sampling -- can seem effective if evaluated only after distillation, while additional reinforcement learning can lead to the defense being broken. \Cref{sec:rl_makes_simple_attacks_work} shows that threat models including RL render very simple attacks effective, allowing reasoning capabilities to be copied using reasoning trace summaries and answers given by closed-source models. Finally, in \Cref{sec:potential_defenses} we discuss potential defenses and future work.

\textbf{Responsible Disclosure}. This work aims to develop more realistic threat models for existing distillation attacks, rather than to propose new attacks. All results have been disclosed to the relevant model providers (Google, OpenAI, and Anthropic), who have approved sharing these results with the wider community. The authors believe many of these results are likely already known to attackers and hope they may be used to further empower defenders.

\vspace{-5pt}
\section{Prior Work}
\vspace{-3pt}
\textbf{Existing attacks and defenses.} All known prior work studying distillation attacks has evaluated attacks, defenses, and mitigations after distillation, but without any subsequent reinforcement learning, both in academic studies \citep{antidist_sampling,antidis_fingerprinting,zhang2026stealreasoningreasoningtraces,doge_defense,libon2026what,allouah2026distillation} and in evaluations by frontier labs \citep[Sec.~5.1.1]{anthropic2026riskreport}. \citet{libon2026what} discuss how different kinds of threat models can change whether a distillation defense is effective, but without assuming any additional training after the distillation. \citet{zhang2026stealreasoningreasoningtraces} demonstrate an attack to steal reasoning capabilities from closed-source models by training an expander to map summaries back into full traces. The attack in \Cref{sec:rl_makes_simple_attacks_work} is simpler, as the expander requires no training, and more importantly seems similar to the existing, reported attacks, based on the little information that is publicly known \citep[see ``Illicit distillation and scaled abuse'']{anthropic2026countering}. \citet{panfilov2026stealing} show that the encrypted or signed reasoning that closed-source APIs return alongside an answer can be replayed into a cheaper decoder model from the same family, which then reads the capable model's hidden reasoning out verbatim.

\textbf{Distillation versus RL.} One reason why previous works assume no reinforcement learning is done after distillation may be due to there being much confusion in the literature on the comparison between distillation and RL. Some works argue that distillation typically outperforms RL \citep{leap, distill-vs-rl, efficient} while others argue that RL typically outperforms distillation \citep{rl-generalizes, math-reasoning-improve}, and some work observes both results in different model sizes \citep{deepseek}. Specific model families, like some Qwen models, are known to lead to conclusions regarding RL that do not generalize to other model families \citep{shao2026spuriousrewardsrethinkingtraining}. To ensure our conclusions are robust, we test methods either across several model families or over a model family that is known to not exhibit spurious performance improvements.

There are also some prevalent intuitions arguing why one would expect distillation to outperform reinforcement learning. Because distillation has a denser learning signal than the binary rewards used in reinforcement learning from verifiable feedback, one would expect distillation to lead to larger improvements than reinforcement learning \citep{lecun-cake,lora}.

Regardless of this confusion, many existing LLM training pipelines combine distillation followed by RL, often in sequence several times \citep{deepseek-nature,qwen, sft-then-rl, efficient, deepcoder2025,deepscaler2025, chem}. For example, \citet{deepseek-nature} discuss successively using distillation to bootstrap their model, training with reinforcement learning, and then using the RL-trained model to generate better data for further bootstrapping, performing overall three cycles of distillation followed by RL. In a realistic distillation attack, data attained from a closed-source model would likely supplement or entirely replace the data used for bootstrapping.

\vspace{-7pt}
\section{Threat Models for Realistic Distillation Attacks}
\vspace{-5pt}
\label{sec:rl_in_threat_model}
Distillation attacks are believed to be carried out by rival industrial labs that desire to train a state-of-the-art LLM. Presumably, an attacker would use all the resources at their disposal, such as API calls or GPU hours, to train the best-performing model possible. Prior work evaluating defenses against attacks evaluates the attacker's model immediately after distillation. This threat model implicitly assumes no further training after distillation, or that defenses would persist throughout any additional training. In this section, we argue that a more realistic threat model includes reinforcement learning after the distillation.

It makes sense to assume that a capable attacker executing a distillation attack performs no further training after distillation only if alternatives to distillation perform.\footnote{This assumes performance is the only metric that matters -- one may wish to train a model further for other purposes, e.g., giving friendlier answers.} The two main methods to endow a model with reasoning capabilities are distillation \citep{s1} and reinforcement learning \citep{deepseek-nature}, yet from the literature, which performs best is unclear. To see if a realistic distillation attack likely uses any training after the distillation, we train a set of models using either distillation or RL and compare which performs best.

In distillation, typically a large, capable ``teacher'' model is distilled into a smaller, less capable ``student'' \citep{distillation}. In a distillation attack, the teacher is the closed-source model while the student is the attacker's model. Throughout the paper, teacher/student models always analogously refer to the closed-source/attacker's model in an attack.

\textbf{Technical details.} Unless specified otherwise, all experiments use the datasets, prompts, and RL framework of \citet{simplerlzoo}, which does RL training using GRPO \citep{grpo}. For fairness, both distillation and RL training use the same sets of questions. We evaluate existing model checkpoints from \citet{simplerlzoo} where available; otherwise, we train models using the same framework. Distillation with subsequent RL is performed only for models with at most 3B parameters, as compute limitations prevented RL training on larger models. Distillation is done using sequence-level knowledge distillation \citep{seq-kd}, which corresponds to supervised fine-tuning over the teacher's generated answers, as is common practice \citep{s1, deepseek}. Throughout the paper, we denote base models trained using RL by appending ``-RL'' to their names: for example, the base model Qwen2.5-14B trained with RL is Qwen2.5-14B-RL. For all distilled models, the teacher used to generate traces for distillation is Qwen2.5-14B-RL, except for Qwen2.5-0.5B, where the teacher is Qwen2.5-1.5B-RL. Additional details and ablations are discussed in Appendix \ref{app:distill-rl}.

\begin{figure}
    \centering
    \includegraphics[width=\linewidth]{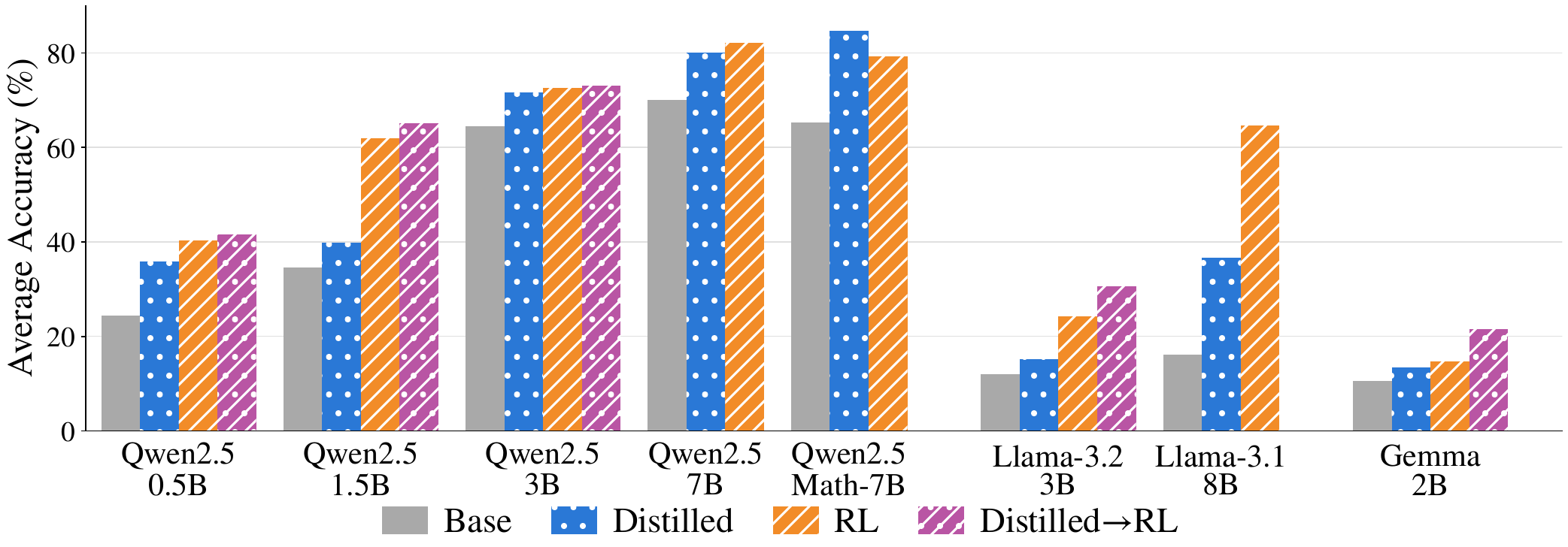}
    \vspace{-18pt}
    \caption{\textbf{RL outperforms distillation, with distillation followed by RL outperforming both.} An attacker aiming to train a state-of-the-art model would likely use distillation to bootstrap subsequent RL, rather than relying solely on distillation. Accuracies are averages over the GSM8K \citep{gsm8k}, Minerva Math \citep{minverva_math}, and MATH500 \citep{math500} datasets.}
    \label{fig:distill_v_rl}
    \vspace{-15pt}
\end{figure}

\textbf{Results.} \Cref{fig:distill_v_rl} shows the performance of various language models when trained using reinforcement learning or distillation over the same dataset. Surprisingly, in almost all cases, distillation underperforms reinforcement learning.\footnote{The one case where RL underperforms distillation, the Qwen2.5 Math 7B model, is an outlier that is analyzed in Appendix \ref{app:explain-outlier}.} Most notably, distillation followed by reinforcement learning outperforms either training method separately. In this sense, distillation bootstraps and improves subsequent reinforcement learning. Thus, an attacker desiring state-of-the-art performance would likely use distillation followed by RL.

\begin{figure}[!b]
    \vspace{-15pt}
    \centering
    \includegraphics[width=\linewidth]{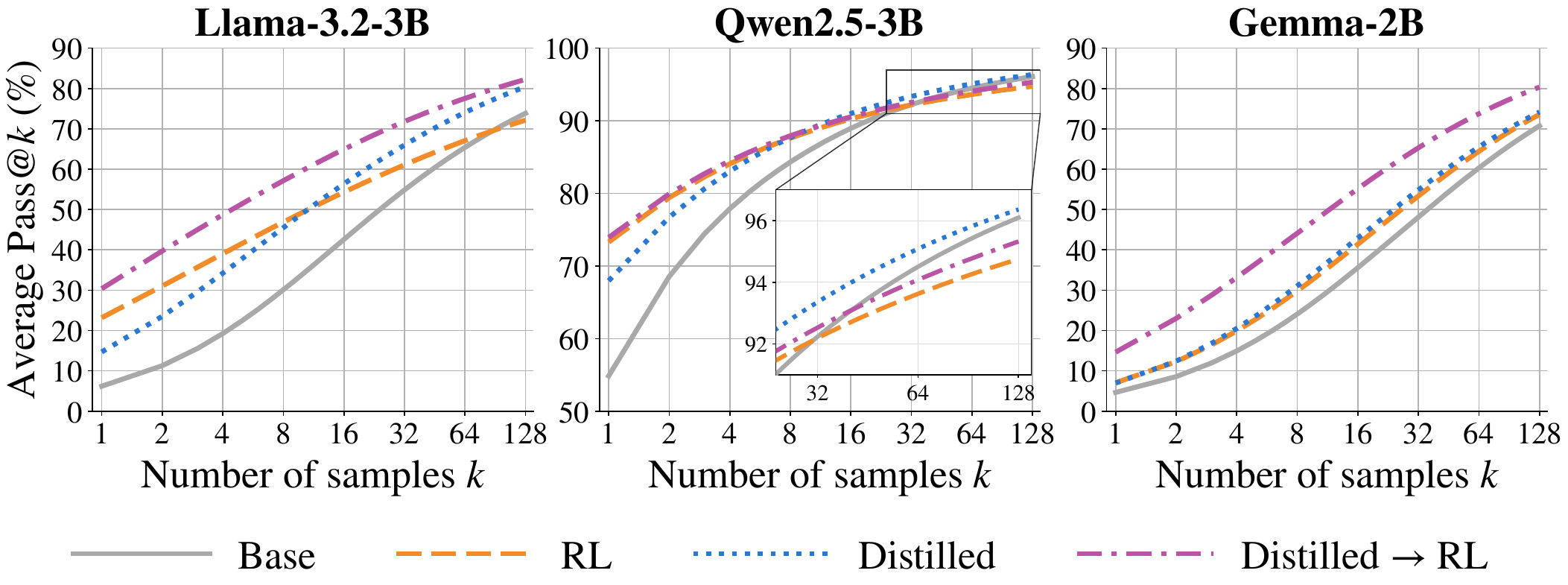}
    \vspace{-16pt}
    \caption{\textbf{Distillation improves pass@$k$ for high $k$, allowing futher RL to lead to accuracy improvements (pass@1)}. In this sense, distillation bootstraps subsequent RL training. Pass@$k$ is averaged over the GSM8K, Minerva, and MATH500 datasets. A detailed breakdown of results is available in Appendix \ref{app:more-details-rl-vs-distill}.}
    \label{fig:pass_at_k}
    \vspace{-9pt}
 \end{figure}
\textbf{Distillation bootstraps RL.} A more informative metric than a model's accuracy for how different kinds of training affect its output distribution is pass@$k$. Pass@$k$ measures the probability that a model answers a question correctly within $k$ attempts \citep{passatk}. Pass@1 is hence a model's accuracy, while pass@$k$ gives more weight to problems that require many attempts to answer correctly. For high $k$, pass@$k$ can be interpreted as a soft performance ceiling or ``reasoning boundary'' \citep{leap}, indicating which problems a model can solve given a large but finite number of attempts. Parts of the following analysis exist in many different forms in prior works \citep{leap,deepseek-nature}, with it being included here only to give additional intuition.

\Cref{fig:pass_at_k} shows that although RL outperforms distillation in improving base-model accuracy (pass@1), distillation outperforms RL in raising a model's soft performance ceiling (pass@$k$). Distillation followed by RL outperforms all methods in improving accuracy, while sometimes also achieving a higher pass@$k$. 

\textbf{Theory.} There are theoretical arguments for why one would expect RL to outperform distillation and why distillation is expected to increase a model's pass@$k$ while RL may be expected to decrease it. Such theory is useful both in helping to understand previous results and in giving evidence for why one would expect to see similar trends in other settings, e.g., larger models. We describe the intuition behind the theory here, with exact results and a full discussion in Appendix \ref{app:theory}.

\begin{wrapfigure}{r}{0.6\linewidth}
    \vspace{-16pt}
    \centering
    \includegraphics[width=\linewidth]{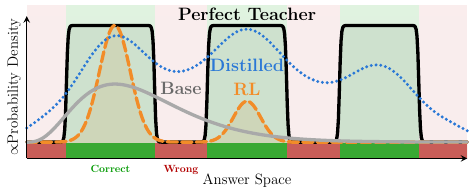}
    \vspace{-24pt}
    \caption{Relative to a \textbf{perfect teacher}, \textcolor{rlColor}{\textbf{RL}} (dashed) is mode-seeking while \textcolor{distilledColor}{\textbf{distillation}} (dotted) is mode-covering. Green regions represent correct answers, yellow represents incorrect. Probability densities are presented unnormalized. In practice, RL is more limited by the initial distribution of the \textcolor{baseTextColor}{\textbf{base}} model being trained than distillation.}
    \label{fig:mode_behavs}
    \vspace{-10pt}
\end{wrapfigure}
When viewed from a probabilistic perspective, language models define a distribution over possible outputs. The cross-entropy loss used in distillation is mode-covering, where the loss highly penalizes the model for giving a low probability when the teacher gives a high probability. This training signal leads the distilled student model to spread its probability mass more broadly, increasing the probability of previously unlikely answers, but potentially leading to more mistakes (see distillation (dotted) curve in Figure \ref{fig:mode_behavs}). RL, on the other hand, is known in some cases to be mode-seeking \citep{levine2018reinforcement}, which concentrates probability mass on correct answers while potentially making some low-probability solutions even lower.

\vspace{-3pt}
\section{RL-Free Evaluations Can Give a False Sense of Security}
\vspace{-3pt}
\label{sec:false_security}
Defenses against distillation attacks are typically evaluated by measuring a distilled model's performance when distilled on a teacher model's traces, with and without the defense being employed. A defense is deemed effective if it meaningfully degrades the attacker model's performance. Here we show that RL-free evaluations can sometimes create a false sense of security: defenses that seem effective after distillation can be ineffective after a model is subsequently trained with RL, with performance gaps vanishing after reinforcement learning.

To illustrate, we evaluate antidistillation sampling \citep{antidist_sampling} as a case study, both after distillation and after distillation followed by RL. Antidistillation sampling is a well-known defense against distillation attacks that adversarially perturbs a teacher model's outputs to degrade a distilled student's performance on a chosen downstream task, thereby ``poisoning'' the attacker. Higher levels of poisoning further degrade the student, at the cost of also degrading the teacher's outputs.

We evaluate Qwen2.5-0.5B \citep{qwen-2.5} distilled on traces generated by a teacher using different levels of poisoning from antidistillation sampling. We use the Qwen2.5-1.5B-RL model as the teacher. Low, mild, and high levels of poisoning are used, which effectively reduce the teacher's relative performance by 10\%, 30\%, and 86\% on the dataset the student is distilled over. These poisoning levels are all aggressive, as in practice even a 10\% relative reduction in teacher performance would likely be too detrimental for the defense to be deployed.

\begin{figure}[h]
    \centering
    \includegraphics[width=\linewidth]{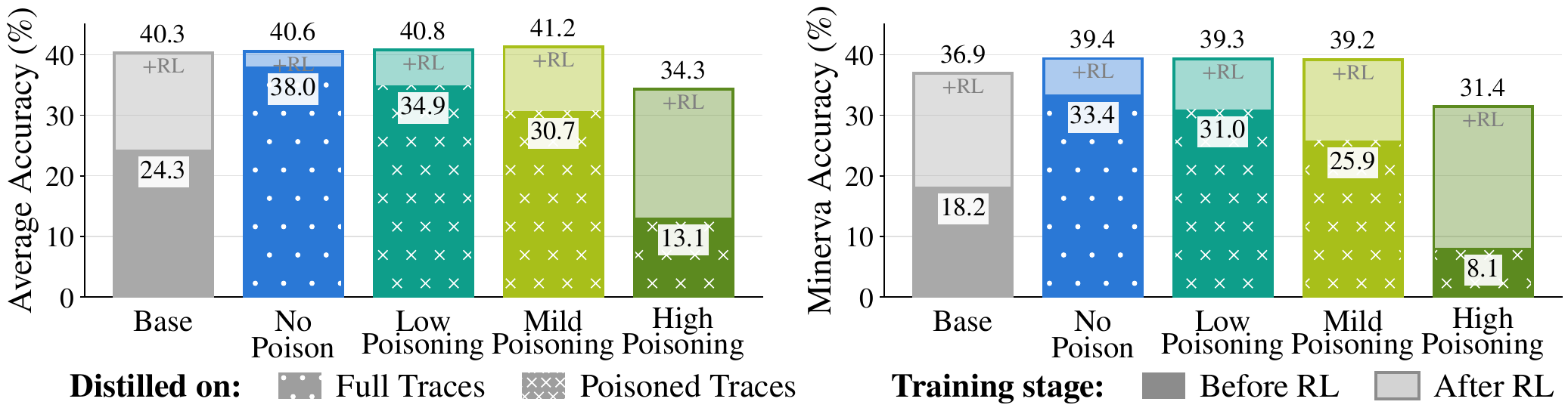}
    \vspace{-15pt}
    \caption{\textbf{Antidistillation sampling can be effective following distillation but break following RL.} Antidistillation sampling modifies teacher traces to ``poison'' distilled students, degrading their performance after distillation. A Qwen2.5-0.5B student is distilled normally and with low, mild, or high poisoning levels; after RL, low- to mildly poisoned student models close the performance gap with the unpoisoned model. An attacker using RL would benefit similarly from antidistillation-sampled data as from regular data. Left: the average accuracy over the GSM8K, MATH500, and Minerva datasets. Right: the performance over Minerva, which is harder than GSM8K and similar in difficulty to MATH500. Additional details are provided in Appendix \ref{app:rl-free}.}
    \label{fig:antidist_qwen05}
    \vspace{-11pt}
\end{figure}

\Cref{fig:antidist_qwen05} shows how antidistillation sampling degrades the student's performance following distillation relative to the unpoisoned student. However, after reinforcement learning, the unpoisoned and low to mildly poisoned students have almost the same performance. Notably, this performance is higher than what the attacker would have achieved if reinforcement learning was done without any distillation. Sufficiently high levels of poisoning lead to degradations that persist after reinforcement learning but are effectively due to the student having a very bad teacher -- see Appendix \ref{app:illustrate-traces} for qualitative examples.

\vspace{-7pt}
\section{RL Makes Simple Distillation Attacks Effective}
\vspace{-5pt}
\label{sec:rl_makes_simple_attacks_work}
\looseness=-1 The previous section showed that a distillation defense, when evaluated out-of-the-box, can seem effective after distillation, while being ineffective after subsequent reinforcement learning. This section shows that even for stronger defenses, simple attacks that seem ineffective after distillation can be effective after RL. We demonstrate this by distilling reasoning capabilities from existing deployed closed-source language models, namely Claude Sonnet 4.6, GPT-5 mini, and Gemini Flash 3.6.

Most deployed closed-source language models protect reasoning data by separating it into a hidden reasoning trace and final user-facing answer. To defend against distillation attacks on reasoning traces while still providing transparency into models' deliberation mechanisms, proprietary APIs return only a summarized version of the model's reasoning to the user \citep{anthropic2026riskreport, openai2026reasoning}, while the model's final answer is returned as-is. Intuitively, as the summary condenses the full reasoning trace's semantic content, a weak off-the-shelf model should be able to leverage the summaries and final answers to reconstruct semantically similar full reasoning traces. For example, it is much easier to solve a math problem given hints on useful steps than from scratch.

\begin{figure}[!ht]
    \centering
    \vspace{-5pt}
    \includegraphics[width=\linewidth]{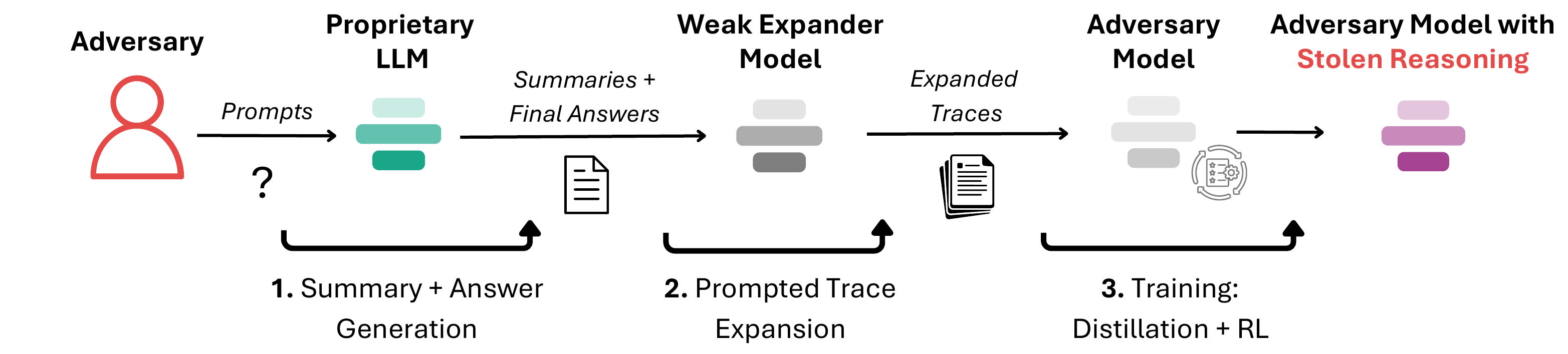}
    \vspace{-12pt}
    \caption{\textbf{A simple attack to approximately reconstruct closed-source models' reasoning traces.} Note that the reconstructed traces need not be faithful to the hidden, typically unknown, full reasoning traces, but only similarly useful for bootstrapping a model's reasoning.}
    \vspace{-4pt}
    \label{fig:attack}
\end{figure}

Building on this intuition, we evaluate the following attack, illustrated in \Cref{fig:attack}. A model's summarized reasoning and answers can be given to a weak ``expander'' model which attempts to reconstruct full reasoning traces. The attacker's model is then distilled on these expanded traces in lieu of the closed-source teacher's hidden reasoning. If the attack is successful, after distillation and reinforcement learning, the attacker's model should achieve similar performance regardless of whether it was distilled on reconstructed, expanded traces or on the original, typically hidden traces.

We first test this attack in an open-source setting where the full reasoning traces are available, so we can easily compare the model's performance when distilled over expanded traces versus the original, full reasoning traces. We use Qwen2.5-14B-RL as the teacher and distill it into Llama-3.2-3B-Base. Llama is chosen as the attacker's model, as in realistic distillation attacks, models come from different model families. Additionally, it is harder to get spurious performance improvements from RL training on Llama models than for models from some other families \citep{shao2026spuriousrewardsrethinkingtraining}, such as Qwen. We use Llama-3.2-3B-Instruct as the weak expander, which notably is not trained using reinforcement learning. Summaries are obtained by prompting Qwen2.5-7B-Instruct to summarize the teacher's reasoning traces, with an in-context example to ensure each reasoning step is described in natural language and includes key results, but no full mathematical derivations. Additional details are available in Appendix \ref{app:attack-details}, as well as a qualitative example of a trace and its summary in Appendix \ref{app:illustrate-traces}.

\Cref{fig:distill_open} shows that although after distillation the performance is lower over the expanded traces, following reinforcement learning there is essentially no gap relative to training on the original full traces. As distillation attacks likely target difficult problems where the closed-source teacher is much more capable than the attacker's student, we chiefly compare results over hard datasets, namely Minerva and MATH500, without GSM8K. Empirically, distilling over the expanded summaries sometimes leads to more forgetting on GSM8K than distilling over the full traces -- see Appendix \ref{app:ablations} for details.

\begin{figure}[!ht]
    \centering
    \vspace{-6pt}
    \includegraphics[width=\linewidth]{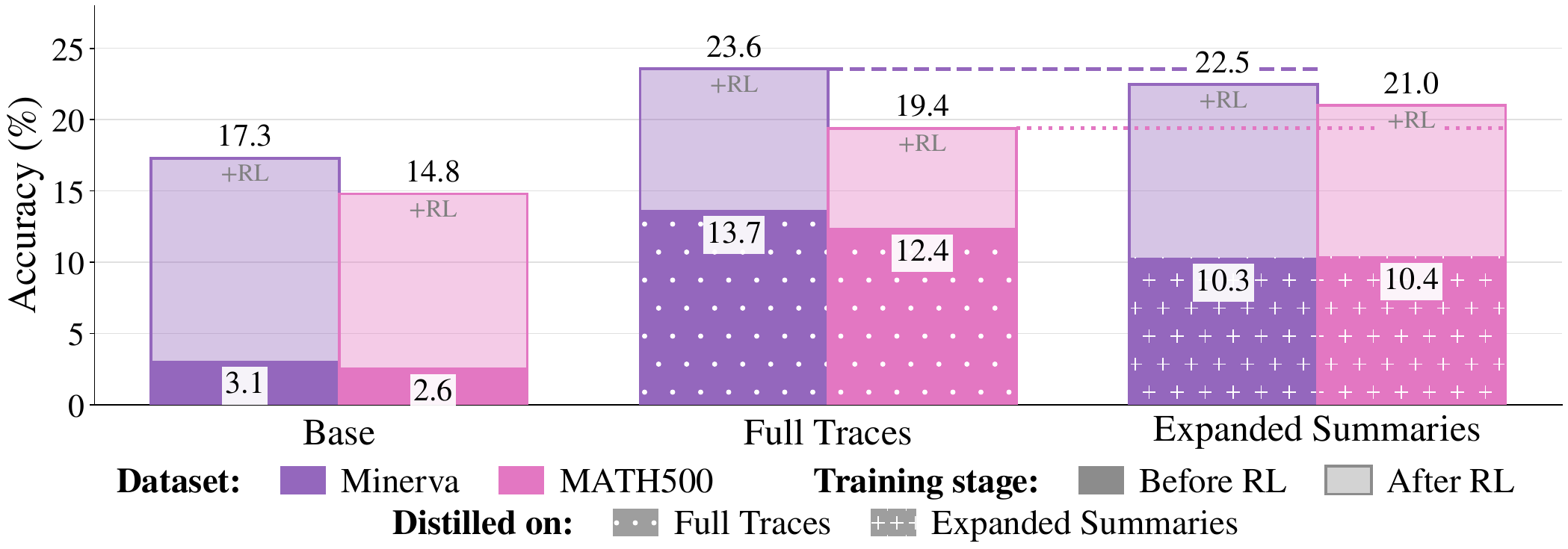}
    \vspace{-10pt}
    \caption{\textbf{Summaries leak sufficient information to distill reasoning capabilities equal to those achieved using full traces.} After reinforcement learning, the model distilled on expanded summaries (right) performs similarly to the model distilled on full, unobfuscated reasoning traces (middle). Distilling on either full traces or expanded summaries decently outperforms no distillation (left). Open-source setting, with Qwen2.5-14B-RL as the attacked teacher, and Llama-3.2-3B base as the attacker's model.}
    \label{fig:distill_open}
    \vspace{-10pt}
\end{figure}

\subsection{Stealing Closed-Source Reasoning}
We use a similar attack against Claude Sonnet 4.6, GPT-5 mini, and Gemini Flash 3.6, with some minor changes. Summaries are obtained directly as part of the response given by the model APIs. Summaries from all three model providers are stylistically different, with Claude summaries being the most concise, GPT-5 mini summaries the most verbose, and Gemini's having many personified expressions and verification. To ease processing, we first prompt the expander to stylistically rephrase the answers and summaries, before querying it again with the homogenized summaries to generate approximate reasoning traces. We forgo attacking the model providers' most capable models due to being unable to RL train more capable LLMs than those with 3B parameters and due to safety concerns, so we do not demonstrate an attack against frontier models. Additional technical details and ablations are discussed in Appendix \ref{app:attack}.

To test whether distilling on expanded traces performs similarly to distilling on full traces, we extract full traces from Claude Sonnet 4.6 and GPT-5 mini using a disclosed variant of the attack demonstrated by \citet{panfilov2026stealing}. \citet{panfilov2026stealing} describe an attack to extract unsummarized reasoning traces from a closed-source language model by jailbreaking a weaker model from the same family and asking it to output some encrypted reasoning. Gemini models had the extraction attack patched when the experiments were performed, and are thus excluded. All results were responsibly disclosed to the relevant model providers. See Appendix \ref{app:attack-details} for additional details.  

\Cref{fig:distill_closed} shows that, following reinforcement learning, the simple trace expansion attack recovers performance similar to that given by full traces from the closed-source models. Although Gemini does not have a full-trace baseline, results from the open-source and two other closed-source settings indicate that the simple attack would likely be effective against Gemini as well.

\begin{figure}[h]
    \vspace{-5pt}
    \centering
    \includegraphics[width=\linewidth]{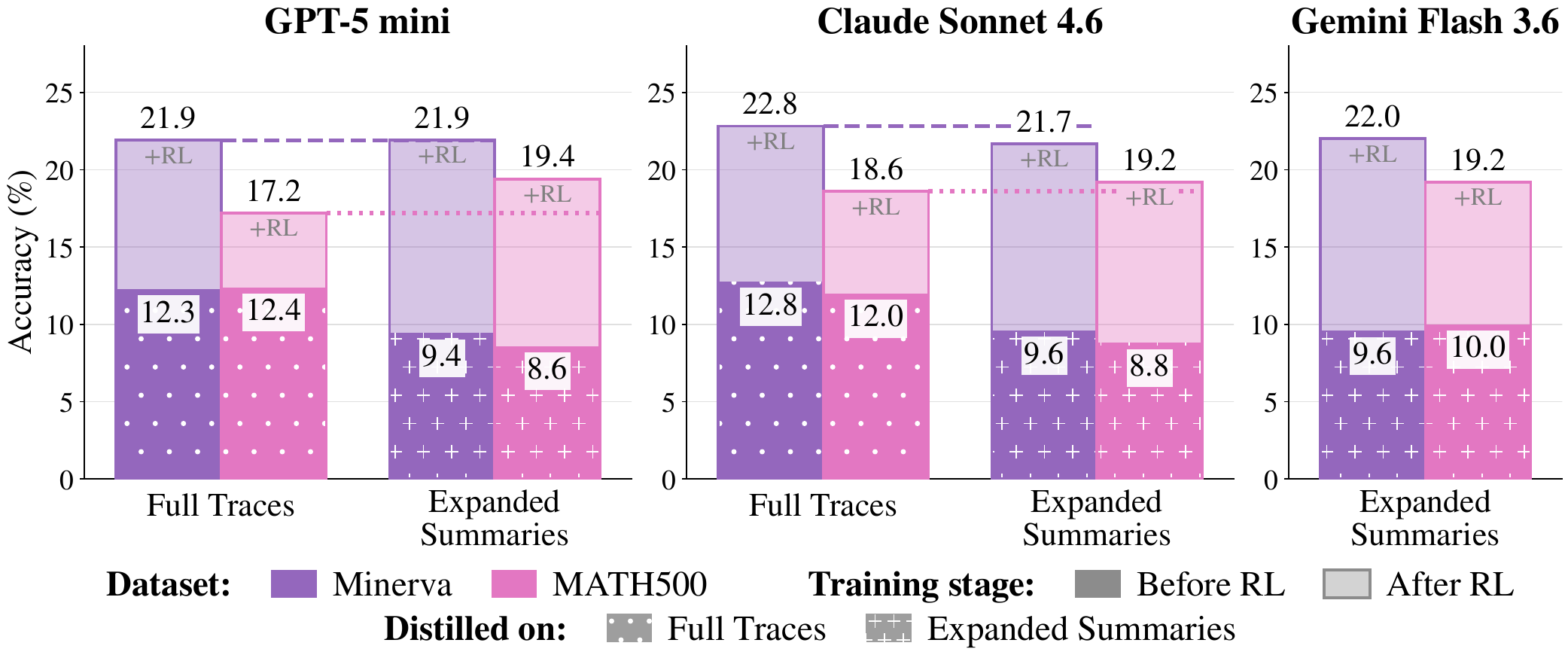}
    \vspace{-15pt}
    \caption{\textbf{Summaries leak sufficient information to distill reasoning capabilities from closed-source models.} Full traces were obtained with the extraction attack of \citet{panfilov2026stealing} (Appendix \ref{app:extraction}), excluding Gemini models, which had the extraction attack patched at the time of writing. After reinforcement learning, a base model distilled on expanded summaries performs similarly to the same model distilled on full traces. Expanded summaries are constructed using information readily available through the model APIs.}
    \label{fig:distill_closed}
    \vspace{-10pt}
\end{figure}

\vspace{-5pt}
\section{Potential Defenses}
\vspace{-3pt}
\label{sec:potential_defenses}
In this section, we discuss which kinds of defenses could be effective against realistic distillation attacks and should be developed in future work.

\vspace{-3pt}
\subsection{Response-level Defenses}
\vspace{-2pt}
Summarizing reasoning traces and antidistillation sampling are both examples of \textbf{broad response-level defenses}, where the defense is applied over a single API call regardless of the domain. As far as the authors are aware, the only real-time defenses currently deployed by model providers are at the response-level. Existing defenses are likely ineffective, evidenced not only by this work but more broadly by the numerous reports of successful distillation attacks throughout 2026 \citep{google,anthropic,anthropic2026countering}. Based on the little publicly released information, it is plausible that attackers have been using attacks similar to the one disclosed here \citep[see ``Illicit distillation and scaled abuse'']{anthropic2026countering}.

More broadly, response-level defenses are likely incapable of preventing distillation attacks due to dual-use. If processed model outputs contain the same semantic information as the full reasoning traces, then it should be possible to reconstruct approximate full traces. The more information is given in the processed traces, the easier it is to reconstruct full traces using simple attacks like the one disclosed here. Reconstructed traces need only be sufficiently useful to teach a model how to solve harder problems than it currently can, as subsequent reinforcement learning teaches the model to do so more reliably.

While it is possible to remove information from a model's output, this would directly lead to a worse user experience, making such a defense inviable. For example, at the response-level, it is difficult to differentiate between a researcher asking for steps on how to prove a theorem versus an attacker extracting data to improve their model's theorem-proving capabilities. Omitting steps of the proof would require more effort from the attacker while worsening the researcher's experience. In most cases, worsening a user's experience for a safer model is likely impractical due to the competition among model providers -- users would prefer a worse-defended but better model over a model which is well defended but harder to use.

Defenses that try to adversarially obfuscate a model's outputs without harming performance would likely be ineffective. One such defense is antidistillation sampling. We found that levels of antidistillation that reduce the distilled model's performance typically also significantly degrade the teacher, making antidistillation impractical to deploy -- see Appendix \ref{app:llama-anti}. More broadly, methods relying on subtle patterns in generated text are known to be brittle, often breaking when the text is rephrased. This has been widely studied in the context of watermarking text generated by language models \citep{sadasivan2025can}.

There are, however, \textbf{\textit{domain specific} response-level defenses} which are sensible and already being deployed. For example, \citet{anthropic2026fablemythos} have safeguards preventing Claude Fable 5 from answering cybersecurity queries, which have reportedly foiled an attempted distillation attack where a rival model provider wished to use Fable to improve their model's cyber capabilities \citep[GTG-16006]{anthropic2026countering}. For the given domain, the tradeoff between security and model usability is sensible.

Moreover, summaries and other methods may still be useful as \textbf{response-level mitigations}. Summaries are reported to prevent models from leaking sensitive information such as memorized passwords and API keys \citep{panfilov2026stealing}, and still require an attacker to postprocess the attacked model's outputs.

\vspace{-5pt}
\subsection{Batch-level Defenses}
\vspace{-3pt}
While it is difficult to tell whether a single query is part of a distillation attack, it might be easier given a set of related queries. Based on available reports, existing \textbf{batch-level defenses} likely only detect attacks currently in hindsight and do not prevent them in real time \citep{anthropic}. How to develop effective real-time batch-level defenses is nontrivial, as attackers are reported to use multiple accounts operating from different locations \citep{anthropic}, making it difficult to detect related queries. However, developing such defenses is a very important avenue for future work.

More broadly, batch-level defenses are likely relevant for a wide class of vulnerabilities, beyond distillation attacks. There have been reported attacks that use a closed-source model to improve a rival model without distillation, e.g. by using the closed-source LLM as a reward model for RL \citep{anthropic}. Boundary Point Jailbreaking (BPJ, \citep{davies2026boundary}) is an attack that automatically finds jailbreaks for a closed-source LLM using many queries. \citet{davies2026boundary} also advocate for batch-level defenses, as response-level defenses struggle to protect against their attack. Thus, if successfully implemented, batch-level defenses could protect against a large class of open vulnerabilities. 

\vspace{-6pt}
\section{Limitations}
\vspace{-3pt}
Several uncertainties remain regarding our findings. All experiments use relatively small models, while realistic distillation attacks likely use models with orders of magnitude more parameters. However, there have been documented cases of distillation attacks against smaller closed-source models, including those studied here \citep{cisa2026distillation}. Even in larger settings, it is likely that a realistic threat model includes reinforcement learning after the distillation, and studies training larger models using distillation followed by reinforcement learning provide evidence that distillation helps bootstrap the subsequent RL \citep{deepseek-nature,qwen,efficient}. Moreover, information in reasoning traces is leaked by summaries regardless of model size; whether that information is qualitatively different between small and larger models in a way that would affect distillation attacks remains to be shown.

Similarly, all experiments focus on math reasoning, while realistic distillation attacks are likely also over other tasks with verifiable solutions, such as long-horizon agentic coding. Because summaries leak substantial information also in agentic setups, with the model producing several intermediate outputs instead of a single final answer, we believe attacks akin to the one proposed here would be similarly viable. Such attacks may already have been used, as attacks have been reported that aim to distill capabilities from existing agentic setups \citep{anthropic,anthropic2026countering}, where the available information is only summaries and intermediate answers.

The attack proposed here is unoptimized, and we do not expect it to be in any sense optimal. It should only be taken to show that simple attacks, given readily available information from the APIs and a realistic threat model, are likely sufficient to steal reasoning capabilities. More performant attacks would potentially yield better reasoning improvements or require less from the attacker. The postprocessing given by the expander model is both computationally and economically cheap, certainly compared with additional reinforcement learning and likely relative to other ways of acquiring more high-quality training data. Thus, while it is possible to reduce the postprocessing, there is likely no good reason to do so.

\subsection*{Reproducibility statement}
Details necessary to reproduce all experiments are given in the main text and throughout the appendices. All code is available at \url{https://github.com/sjavaheri/treason}, except for the code used to extract full reasoning traces from closed-source models.

\subsubsection*{Acknowledgments}

We would like to express our gratitude to Ilia Shumailov for discussions throughout this work that helped with both research ideation and the framing of results within a wider security context. Discussions with Nick Rhinehart were invaluable in helping us understand why RL was outperforming distillation, the relationship between behavioral cloning and on-policy RL, and led to the theory in \Cref{app:theory}. We are indebted to Xander Davies and Jai Patel of the UK AI Security Institute, who helped us acquire the compute which made much of this project possible. Thanks to Anushka Nair for reviewing a workshop version of this paper. We would like to thank Lena Libon for discussions on different kinds of threat models and connecting the Oxford team with Alexander, which led to a fruitful collaboration.

The authors acknowledge the use of resources provided by the Isambard-AI National AI Research Resource (AIRR) \citep{mcintoshsmith2024isambardaileadershipclasssupercomputer}. Isambard-AI is operated by the University of Bristol and is funded by the UK Government's Department for Science, Innovation and Technology (DSIT) via UK Research and Innovation; and the Science and Technology Facilities Council [ST/AIRR/I-A-I/1023]. We would like to thank Schmidt Sciences for a grant that enabled some of the experiments with closed-source models. Yonatan is funded by the Rhodes Trust and the AIMS EPSRC CDT (grant no. EP/S024050/1).

\bibliography{iclr2027_conference}
\bibliographystyle{iclr2027_conference}

\newpage
\appendix
\startcontents[appendices]
\printcontents[appendices]{}{0}{\section*{Appendix Table of Contents}}

\newpage

\section{Comparing Distillation and Reinforcement Learning}
\vspace{-5pt}
\label{app:distill-rl}

In this section, we discuss the implementation details for the comparison between RL and distillation in \Cref{sec:rl_in_threat_model}, analyze the one outlier that improves more with distillation than with RL (Qwen2.5-Math-7B), discuss some theory on the distillation and RL losses, and present some ablations that act as sanity checks for the results.

\vspace{-5pt}
\subsection{Implementation Details}
\vspace{-5pt}
\label{app:distill-rl-details}

\textbf{Models}. We run experiments on base models of varying sizes from the Qwen2.5 (0.5B, 1.5B, 3B, 7B, Math-7B) \citep{qwen-2.5}, Llama3 (3.2-3B and 3.1-8B) \citep{llama}, and Gemma 1 (2B) \citep{gemma} families. We use the SimpleRLZoo framework for all RL experiments \citep{simplerlzoo}, evaluating existing checkpoints when available, and otherwise training models ourselves. Compute limitations restrict our own RL training to models with at most 3B parameters.

\textbf{Datasets and prompts}. For each base model, we use the same dataset of questions and the same prompts for distillation and reinforcement learning. Following the approach and datasets used by \citet{simplerlzoo}, we train weaker base models on the medium-difficulty SimpleRLZoo dataset with a ``simple'' prompt, while we train more capable models on the harder SimpleRLZoo dataset with a ``complex'' prompt. Weak base models include Qwen2.5-0.5B and all models in the Llama3 and Gemma families. Teacher model traces are always generated with the ``complex'' prompt. Further details on the framework, including the explicit prompts used, are available in \citet{simplerlzoo}. 

\textbf{Distillation hyperparameters}. Distillation experiments and hyperparameters follow the setup of \citet{s1} and are the same for all models. We used a learning rate of $10^{-5}$ with weight decay of $10^{-4}$, along with a cosine learning rate scheduler with a warmup ratio of $0.1$. Gradients are accumulated in steps of 16, with a batch size of 1 per device. All distillation is performed over $1$ training epoch. Teacher traces are generated using a temperature of $T=1$, with one generated trace per question in the dataset, with the distillation being for $1$ epoch.

\textbf{Reinforcement learning hyperparameters}. We run RL training with the SimpleRLZoo framework on either 4 A100 or 4 H100 GPUs with an effective batch size of 1024. All hyperparameters match those of \citet{simplerlzoo}, except for using 4 GPUs instead of 8, reducing the validation batch size from 500 to 256, and reducing the maximum context length to 2048. We found that responses were not truncated and using a longer context length of 8192 did not make a difference.

\textbf{Evaluations}. We evaluate all models on the GSM8K \citep{gsm8k}, Minerva Math \citep{minverva_math}, and MATH500 \citep{math500} datasets using the LM Evaluation Harness \citep{eval-harness}, with all evaluations at $T=0$. We extract model answers and compare them with the correct solutions using the harness's built-in implementation of math-verify \citep{math_verify}, assessing reasoning correctness rather than conflating it with instruction-following and formatting abilities. All evaluations allowed LLM responses up to 2048 tokens. 

Pass@$k$ experiments are done with a temperature of $T=1$ on all datasets, using the Python implementation of math-verify to check the correctness of model responses \citep{math_verify}. Pass@$k$ calculations are done by sampling 256 answers to each question for each model, ensuring that results at the maximum value of $k=128$ are reliable \citep{passatk}.

\vspace{-5pt}
\subsection{Per-Dataset Breakdown}
\vspace{-5pt}
\label{app:more-details-rl-vs-distill}

\Cref{tab:distill-rl-full} shows a per-dataset breakdown of the results comparing RL and distillation in \Cref{fig:distill_v_rl}. \Cref{fig:pass_at_k_all_datasets} shows that the average pass@$k$ curves in \Cref{fig:pass_at_k} hold on the individual datasets as well. 

\begin{table}[!h]
\caption{\textbf{Per-dataset breakdown for results in \Cref{fig:distill_v_rl}.} RL outperforms distillation, with distillation followed by RL outperforming only distillation and only RL.}
\vspace{-10pt}
\label{tab:distill-rl-full}
\begin{center}
\setlength{\tabcolsep}{6pt}
\resizebox{\linewidth}{!}{%
\begin{tabular}[t]{lcccc}
& \multicolumn{4}{c}{\bf Dataset} \\
\cmidrule(l){2-5}
{\bf Model} & {\bf GSM8K} & {\bf Minerva} & {\bf Math500} & {\bf Average} \\
\midrule
\multicolumn{5}{c}{\em Qwen2.5-0.5B} \\
\midrule
Base            & 35.6\% & 18.2\% & 19.2\% & 24.3\% \\
Base-RL         & 48.5\% & 36.9\% & 35.4\% & 40.3\% \\
Distill         & 45.1\% & 30.7\% & 31.6\% & 35.8\% \\
Distill-RL      & 49.1\% & 40.2\% & 35.4\% & 41.6\% \\
\midrule
\multicolumn{5}{c}{\em Qwen2.5-1.5B} \\
\midrule
Base            & 49.9\% & 27.5\% & 26.2\% & 34.5\% \\
Base-RL         & 74.5\% & 60.2\% & 59.0\% & 64.6\% \\
Distill         & 72.3\% & 56.8\% & 54.2\% & 61.1\% \\
Distill-RL      & 75.4\% & 60.5\% & 59.6\% & 65.1\% \\
\midrule
\multicolumn{5}{c}{\em Qwen2.5-3B} \\
\midrule
Base            & 75.1\% & 60.0\% & 58.4\% & 64.5\% \\
Base-RL         & 84.5\% & 69.9\% & 63.4\% & 72.6\% \\
Distill         & 82.5\% & 66.1\% & 66.2\% & 71.6\% \\
Distill-RL      & 83.0\% & 70.0\% & 66.4\% & 73.1\% \\
\midrule
\multicolumn{5}{c}{\em Qwen2.5-7B} \\
\midrule
Base            & 82.3\% & 63.6\% & 64.2\% & 70.0\% \\
Base-RL         & 88.9\% & 78.8\% & 78.8\% & 82.2\% \\
Distill         & 89.0\% & 75.9\% & 75.4\% & 80.1\% \\
Distill-RL      & -- & -- & -- & -- \\
\end{tabular}%
\hspace{1em}%
\begin{tabular}[t]{lcccc}
& \multicolumn{4}{c}{\bf Dataset} \\
\cmidrule(l){2-5}
{\bf Model} & {\bf GSM8K} & {\bf Minerva} & {\bf Math500} & {\bf Average} \\
\midrule
\multicolumn{5}{c}{\em Qwen2.5-Math-7B} \\
\midrule
Base            & 64.7\% & 65.8\% & 65.2\% & 65.2\% \\
Base-RL         & 82.7\% & 77.6\% & 77.6\% & 79.3\% \\
Distill         & 89.0\% & 81.6\% & 83.6\% & 84.7\% \\
Distill-RL      & -- & -- & -- & -- \\
\midrule
\multicolumn{5}{c}{\em Llama-3.2-3B} \\
\midrule
Base            &  0.8\% &  3.1\% &  2.6\% &  2.2\% \\
Base-RL         & 40.7\% & 17.3\% & 14.8\% & 24.3\% \\
Distill         & 33.4\% & 13.7\% & 12.4\% & 19.8\% \\
Distill-RL      & 48.7\% & 23.6\% & 19.4\% & 30.5\% \\
\midrule
\multicolumn{5}{c}{\em Llama-3.1-8B} \\
\midrule
Base            & 21.2\% & 13.8\% & 13.4\% & 16.1\% \\
Base-RL         & 76.2\% & 58.8\% & 59.0\% & 64.7\% \\
Distill         & 59.4\% & 26.2\% & 24.2\% & 36.6\% \\
Distill-RL      & -- & -- & -- & -- \\
\midrule
\multicolumn{5}{c}{\em Gemma2-2B} \\
\midrule
Base            & 10.2\% & 11.0\% & 10.4\% & 10.6\% \\
Base-RL         & 18.1\% & 12.8\% & 13.0\% & 14.7\% \\
Distill         & 15.9\% & 12.9\% & 11.6\% & 13.5\% \\
Distill-RL      & 24.5\% & 20.0\% & 20.0\% & 21.5\% \\
\end{tabular}}
\vspace{-5pt}
\end{center}
\end{table}

\begin{figure}[!h]
    \centering
    \includegraphics[width=1\linewidth]{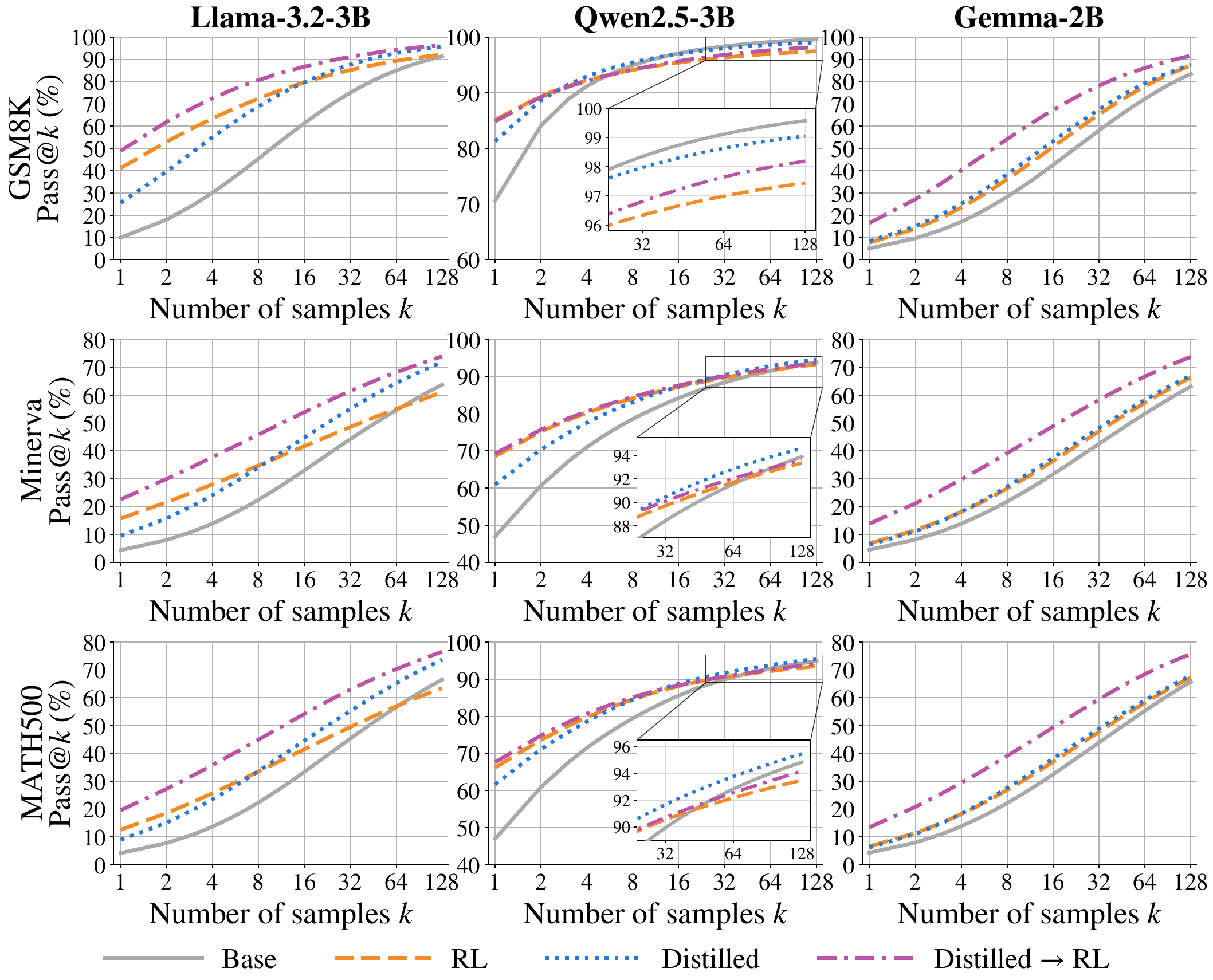}
    \vspace{-15pt}
    \caption{Per-dataset pass@$k$ curves are qualitaitvely similar to their average, shown in  \Cref{fig:pass_at_k}. Results are shown for Llama-3.2-3B (left), Qwen2.5-3B (middle), and Gemma-2B (right), across the GSM8K (top), Minerva (middle), and MATH500 (bottom) datasets.}
    \label{fig:pass_at_k_all_datasets}
    \vspace{-15pt}
    
\end{figure}

\vspace{-5pt}
\subsection{Explaining the Qwen2.5-Math-7B Outlier}
\vspace{-5pt}
\label{app:explain-outlier}

\begin{figure}[!t]
  \centering
  \begin{minipage}[c]{0.48\columnwidth}
    \begin{tracebox}{traceyellow}
      \scriptsize\ttfamily\raggedright
      To find the positive square root of the product \bs(10 \bs times 15 \bs times 24\bs), we can follow these steps: ...

      \op python\\
      import math\\
      \# Calculate the product\\
      product = 10 * 15 * 24\\
      \# Calculate the square root of the product\\
      square\_root = math.sqrt(product)\\
      print(square\_root)\\
      \op
 
      \op output\\
      60.0\\
      \op
 
      The positive square root of the product \bs(10 \bs times 15 \bs times 24\bs) is \bs(\bs boxed\{60\}\bs).
    \end{tracebox}
  \end{minipage}
  \begin{minipage}[c]{0.48\columnwidth}
    \begin{tracebox}{tracered}
      \scriptsize\ttfamily\raggedright
      To determine the number of ways to arrange the letters of the word ``ELLIPSE,'' we need to ...

      \op python\\
      import math\\
      ...\\
      \# Calculate the number of distinct permutations\\
      num\_permutations = math.factorial(n) // math.factorial(freq\_E)\\
      print(num\_permutations)\\
      \op
 
      \op output\\
      2520\\
      \op
 
      The number of ways to arrange the letters of the word ``ELLIPSE'' is \bs(\bs boxed\{2520\}\bs).
    \end{tracebox}
  \end{minipage}\hfill
  
  \caption{The Qwen2.5-Math-7B primarily reasons in Python code and hallucinates its corresponding outputs. This is illustrated with sample reasoning traces of both a correct (left) and an incorrect (right) answer. Traces are left in their raw form for illustration, with some code and reasoning being omitted.}
  \label{fig:example_python}
\end{figure}

\begin{figure}[!th]
    \centering
    \vspace{-5pt}
    \includegraphics[width=\linewidth]{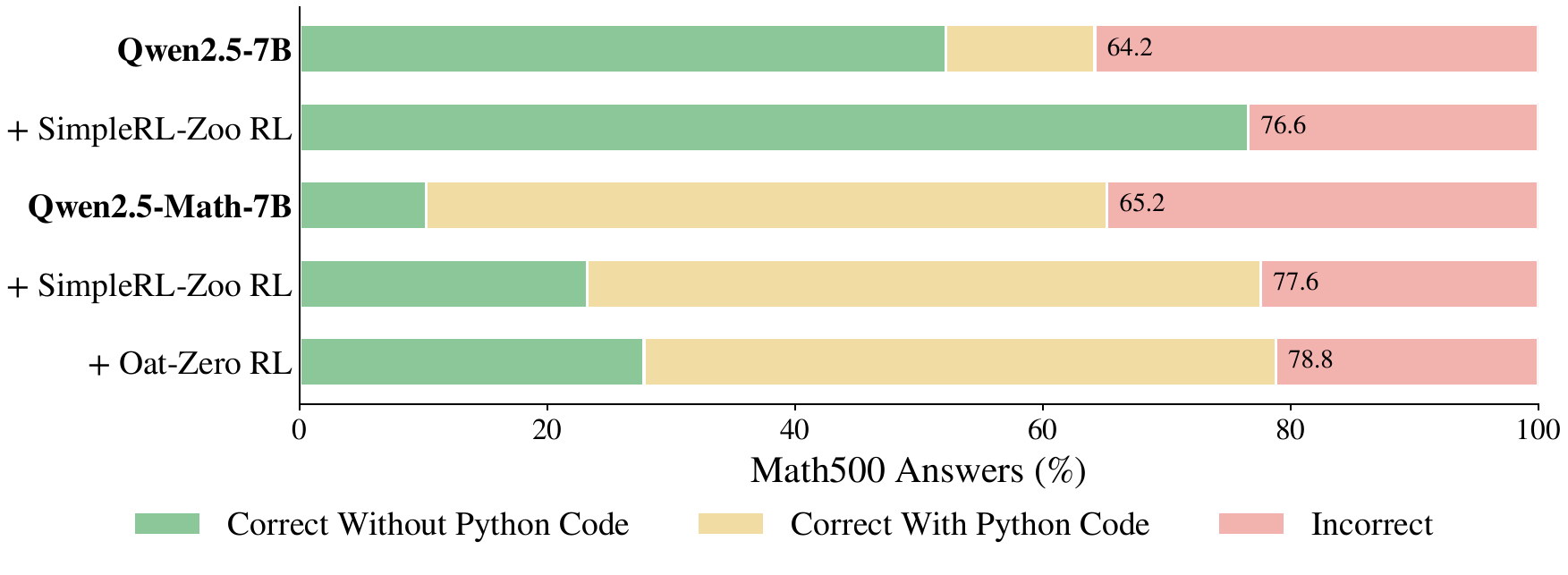}
    \vspace{-10pt}
    \caption{\textbf{RL training fails to eliminate the undesirable pattern of reasoning with Python code when it is the dominant approach used by a base model.} The Qwen2.5-7B base model (top) rarely answers questions correctly using code, so RL training eliminates its code-reasoning behavior. In contrast, Qwen2.5-Math-7B (middle) obtains correct answers predominantly with code, and RL training through two frameworks \citep{oatzero, simplerlzoo} reinforces this code-reasoning behavior.}
    \label{fig:python-dominates}
\vspace{-15pt}
\end{figure}

Distillation outperforms RL only on the Qwen2.5-Math-7B base model \citep{qwen-2.5}. This is the only base model that often exhibits an undesirable reasoning pattern for producing correct answers: it writes Python code, and then hallucinates the code's output, as illustrated in \Cref{fig:example_python}. RL training reinforces this suboptimal behavior, whereas distillation eliminates it, since the teacher's traces contain no examples of using code to answer questions; thus, distillation outperforms RL in this instance. \Cref{fig:python-dominates} demonstrates that the Qwen2.5-7B base model training does not suffer from the same limitation, where most of its reasoning does not include Python. 

\subsection{Theory}
\label{app:theory}
RL outperforming distillation is surprising for several reasons. Sequence-level distillation \citep{seq-kd} is a form of supervised learning, which is typically considered more efficient than RL. Concretely, from an information theory perspective, RL from binary rewards has been argued to provide only $O(1)$ bits per episode, whereas supervised learning gives tens to thousands of bits per sample \citep{lecun2016predictive, lora}. Moreover, some works have empirically shown cases where distillation outperforms RL \citep{deepseek, leap}. Thus, why does RL outperform distillation?

A different perspective that explains these results comes from probabilistic inference; we first give high-level intuition and then some more precise results. First, note that distillation is also a form of RL, corresponding broadly to off-policy reinforcement learning, more specifically to imitation learning, which, in the case of sequence-level knowledge distillation, is simply behavior cloning \citep{rashidinejad2021bridging,foster2024behavior}. The loss used for knowledge distillation uses a forward KL and is therefore mode-covering \citep{mini-llm-kl-divergence}, meaning it puts a high probability over the teacher's modes, potentially interpolating between them (see dotted blue line in Figure \ref{fig:mode_behavs}). This is a general property of imitation learning losses, specifically for a broad class of losses stemming from $f$-divergences \citep{ke2020imitation}. Thus, even given a perfect teacher, some of the student's probability mass could fall outside the teacher's support and therefore on incorrect answers. In contrast, typical max-entropy on-policy RL is mode-seeking relative to a perfect teacher \citep{levine2018reinforcement}, which generally causes an RL-trained student to collapse its probability mass to a subset of correct answers -- see the dashed orange line in Figure \ref{fig:mode_behavs}. 

This intuition can be rigorously formalized. Assuming a temperature of $T=1$, regular knowledge distillation minimizes the cross-entropy between a student and a teacher's distribution, which is equivalent to minimizing the mode-covering KL-divergence $D_{KL}(p_t|p_s)$, with $p_s,p_t$ denoting the student and teacher's output distributions, respectively. However, sequence-level knowledge distillation is not clearly mode-covering, as it first generates a set of completions using the teacher and then trains the student on those completions. This is shown to be a noisy approximation of the regular mode-covering knowledge distillation loss.

\begin{theorem*}
    (Seq-KD is noisy regular KD) Denote the teacher and student model distributions by $p_t,p_s$ respectively. The regular knowledge distillation (KD) loss is $L_{KD}\coloneqq D_{KL}(p_t|p_s)=\mathbb{E}_{p_t}[-\log(p_s)]-H(p_t)$, where $H(p_t)$ denotes the teacher distribution's entropy and we assume a temperature of $T=1$. The sequence-level knowledge distillation (Seq-KD) loss is $\hat{L}_{\text{Seq}-KD}\coloneqq \frac{1}{|D|}\sum_{x\in D}-\log(p_s(x))$, where $D\sim p_t$ is a set of completions sampled i.i.d. from the teacher. The Seq-KD loss' gradient is an unbiased estimate of the KD gradient, such that
    $$\nabla L_{KD}=\mathbb{E}_{D\sim p_t}[\nabla \hat{L}_{\text{Seq}-KD}].$$
\end{theorem*}
\begin{proof}
    Note that:
    $$\mathbb{E}_{D\sim p_t}[\hat{L}_{\text{Seq}-KD}]= \frac{1}{|D|}\sum_{i=1}^{|D|}\mathbb{E}_{x\sim p_t}[-\log(p_s(x))]=\frac{1}{|D|}|D|\mathbb{E}_{x\sim p_t}[-\log(p_s(x))]=\mathbb{E}_{p_t}[-\log(p_s)].$$
    Thus, since $\nabla H(p_t)=0$ as the teacher is constant with respect to the student distribution's parameterization, we have that:
    $\nabla L_{KD}=\nabla \mathbb{E}_{p_t}[-\log(p_s)] = \nabla \mathbb{E}_{D\sim p_t}[\hat{L}_{\text{Seq}-KD}],$
    thereby proving the theorem.
\end{proof}

Thus, the sequence-KD loss, which is typically used when distilling different language models, has the same mode-covering behavior as regular knowledge distillation.

Note that \citet{seq-kd} propose using sequence-KD with beam search and potentially other sampling modifications, which bias the effective teacher's distribution. Here we assume no such modified sampling is used, as in practice, beam search is rarely used with modern LLMs.

In contrast, RL is mode-seeking relative to a specific distribution. Specifically, assume a ``perfect'' teacher that places uniform probability mass over all correct answers. Policy-gradient methods like PPO \citep{schulman2017proximal} and GRPO \citep{grpo} modify the REINFORCE gradient to allow taking off-policy steps, with the regular on-policy REINFORCE gradient being $\mathbb{E}_p[r\nabla\log(p)]$, where $p$ is the policy's distribution and $r$ is the reward. Often an entropy term is added to induce exploration, yielding the loss $\mathbb{E}_p[r\nabla\log(p)]+\beta H(p)$, where $\beta$ is a hyperparameter. When training models to solve math or coding problems by reasoning about them, $r$ is typically a binary reward of $1$ if the task is solved correctly and $0$ otherwise. This reward is thus uniform over all correct answers. This intuition yields the following theorem, which is essentially a different version of a theorem from \citet{levine2018reinforcement}, with some modifications and simplifications for our setting.

\begin{theorem*}
    (Max-entropy RL is mode-seeking to a perfect teacher) Let $p_s$ denote a student model's distribution and $p_t$ be the distribution of a soft perfect teacher, where $p_t(x)\propto \exp(\alpha r(x))$ for some $\alpha$, where $r(x)=1$ (for correct answers) and $0$ otherwise. Note that when $\alpha\to\infty$, then $p_t$ is uniform over correct answers and is a hard perfect teacher. Denote the max-entropy RL policy gradient as $\nabla L_\text{max-ent}=\mathbb{E}_{p_s}[r\nabla\log(p_s)]+\beta \nabla H(p_s)$. Also, denote the mode-seeking (reverse-KL) distillation loss as $L_{\text{rev}-KL}=D_{KL}(p_s|p_t)$. For appropriately chosen $\alpha$ or $\beta$ we have that
    $\nabla L_\text{max-ent}\propto \nabla L_{\text{rev}-KL}.$
\end{theorem*}
\begin{proof}
    Note that:
    $$L_{\text{rev}-KL}=D_{KL}(p_s|p_t)=\mathbb{E}_{p_s}[-\log(p_t)]-H(p_s).$$
    As the perfect teacher's distribution is $p_t(x)=\exp(\alpha r(x))/Z_{\alpha},$ we have that:
    $$L_{\text{rev}-KL}=\mathbb{E}_{p_s}[-\alpha r]-\log(Z_\alpha)-H(p_s).$$
    Setting $\alpha=\frac{1}{\beta}$ and taking gradients completes the proof, as $\nabla Z_\alpha=0$. Note that the sign difference is due to $L_{\text{rev}-KL}$ being a proper loss that is minimized, whereas $L_\text{max-ent}$, being a surrogate loss for a policy gradient, is maximized.
\end{proof}
Thus, as the RL loss is equivalent to a reverse-KL, it induces mode-seeking behavior. Interestingly, the $\alpha\to\infty$ perfect teacher limit, where the teacher is uniform over correct answers, corresponds to $\beta\to0$, where the max-entropy RL becomes regular reward maximization. 

Mode-seeking losses, as used in on-policy RL, enable mode-collapsed policies to achieve low loss, where the probability mass is over some but not necessarily all modes (dashed orange line in Figure \ref{fig:mode_behavs}). However, mode-covering losses, as used in distillation, require a wider support, which can result in significant probability mass between modes being given to incorrect answers (dotted blue line in Figure \ref{fig:mode_behavs}), lowering pass@1 performance of models trained using distillation when compared to those trained using RL. \citet{mini-llm-kl-divergence} discusses similar intuition as well, and show that a mode-seeking on-policy distillation method outperforms mode-covering off-policy distillation.

These losses different behaviors offers a speculative explanation for why distillation consistently increases a model's pass@$k$ for high k, while RL often does not \citep{leap}. Given a subset of problems for which the base model has very little or no probability mass over correct answers, RL training may not increase the probability of generating correct answers, since its mode-seeking loss does not penalize giving no mass to low-probability solutions, as long as some other solutions are decently probable. In contrast, distillation would more readily increase the low-probability mass over correct answers to these problems, since its mode-covering loss incentivizes the student to cover all the modes of its teacher. Thus, when sampled from many times, and when compared to RL-trained models, distilled models are more often able to solve problems they previously could not. 

\subsection{Distillation Ablations}
\label{app:ablations}
We perform several ablations to see whether better distillation setups can make distillation outperform RL. All ablations yield at most marginal performance gains and do not affect qualitative conclusions. Results in \Cref{sec:rl_in_threat_model} use the simplest experimental setup, with no ablations applied.

\textbf{Rejection sampling}. We investigated whether curating higher-quality teacher reasoning traces for SFT with rejection sampling lets distillation outperform RL in improving a model's reasoning, using the Qwen2.5 base models with 0.5B, 1.5B, and 7B parameters. Teachers were given either 8, 32, or 64 attempts to answer each question in the distillation dataset, with only one correct answer per question being kept. Questions with no correct generated answers were discarded. The best result was a marginal performance improvement (roughly $+1\%$ average accuracy) with 8 samples, but it was not sufficient to surpass models trained with RL, as shown in \Cref{tab:rejection-sampling}. Rejection sampling with 32 or 64 attempts per question led to marginally worse results. 

\begin{table}[!h]
\caption{\textbf{Distillation with rejection sampling still underperforms reinforcement learning.} Rejection sampling is done by allowing the teacher 8 attempts at every question, and distillation is then done only over correct-answer traces.}
\vspace{-15pt}
\label{tab:rejection-sampling}
\begin{center}
\setlength{\tabcolsep}{6pt}
\begin{tabular}{lcccc}
& \multicolumn{4}{c}{\bf Dataset} \\
\cmidrule(l){2-5}
{\bf Model} & {\bf GSM8K} & {\bf Minerva} & {\bf Math500} & {\bf Average} \\
\midrule
\multicolumn{5}{c}{\em Qwen2.5-0.5B} \\
\midrule
Base            & 35.6\% & 18.2\% & 19.2\% & 24.3\% \\
Base-RL         & 48.5\% & 36.9\% & 35.4\% & \textbf{40.3}\% \\
Distill         & 47.0\% & 31.5\% & 31.4\% & 36.6\% \\
\midrule
\multicolumn{5}{c}{\em Qwen2.5-1.5B} \\
\midrule
Base            & 49.9\% & 27.5\% & 26.2\% & 34.5\% \\
Base-RL         & 70.9\% & 56.1\% & 59.0\% & \textbf{62.0}\% \\
Distill         & 61.9\% & 33.3\% & 33.4\% & 42.9\% \\
\midrule
\multicolumn{5}{c}{\em Qwen2.5-7B} \\
\midrule
Base            & 82.3\% & 63.6\% & 64.2\% & 70.0\% \\
Base-RL         & 88.9\% & 78.8\% & 78.8\% & \textbf{82.2}\% \\
Distill         & 88.6\% & 75.8\% & 77.6\% & 80.7\% \\
\end{tabular}
\vspace{-10pt}
\end{center}
\vspace{-10pt}
\end{table}

\textbf{Students and Teachers From Different Model Families}. We compare RL and distillation when the teacher's architecture differs from the students' in \Cref{fig:distill_v_rl}, where experiments with Llama and Gemma base models use a Qwen teacher for distillation. RL outperforms distillation here as well.

\textbf{Temperature}. Distillation on teacher traces generated with $T=1$ marginally outperformed distillation on traces generated with $T=0$, so we use $T=1$ to generate all distillation traces. This is supported by theory; sampling completions with temperature 1 is equivalent in the large-data limit to regular knowledge distillation \citep{distillation} (see Appendix \ref{app:theory}). 

Additional ablations were performed with both the Qwen2.5-7B base model (see results in \Cref{tab:distill-ablations}) and the Qwen2.5-0.5B model for faster iteration (see results in \Cref{tab:distill-ablations-05b}). We discuss conclusions below. 

\textbf{Varying Teacher Size}. Using smaller teachers in distillation (where the RL-trained variant of the base model is used as the teacher, and therefore has the same size as the student) marginally improved the student's performance. This aligns with the literature on how a student may struggle to learn from a much more capable teacher \citep{capacity_gap, dist_scale_laws}. However, these performance improvements never led students to exceed the performance of their RL-trained counterparts. For consistency throughout, we primarily choose a larger, more powerful open-source model as the teacher (Qwen2.5-14B-RL), which is also more representative of distillation attacks, where the closed-source model is likely much more capable than the student.

\textbf{Varying Distillation Datasets}. We test how using different question datasets for distillation and RL affects the student's performance. This led to mixed results. Model performance improved slightly (up to $+2\%$) when we used marginally harder questions for distillation on the Qwen2.5-0.5B base model, using the hard rather than medium question dataset from SimpleRL-Zoo \citep{simplerlzoo}. However, using the questions in AIME \citep{aime_1983_2024}, GSM8K \citep{gsm8k}, or Simple Test Time Scaling \citep{s1} datasets to generate teacher traces decreased performance significantly. No dataset made the distilled student outperform RL. In the main paper, for simplicity and fairness, we report results using the same dataset for both distillation and RL.

\textbf{Sampling a Few Times Per Question or Training Over Multiple Epochs}. Using the Qwen2.5-0.5B base model, we attempted distillation over datasets created by sampling from the teacher 3 or 8 times per question, as well as training for 2, 3, or 5 epochs. Increasing either marginally harmed student performance, so we use 1 sample from each question for distillation and train for 1 epoch.

\begin{table}[!h]
\caption{\textbf{Further distillation ablations on Qwen2.5-7B do not qualitatively change results.} Changing the default distillation setup (see \Cref{app:distill-rl-details}) to use a smaller teacher (Qwen2.5-7B-RL \citep{simplerlzoo}) marginally improves performance. Changing the setup either to use a different teacher (Deepseek-R1-Distill-Qwen-14B \citep{deepseek} or Qwen3-14B \citep{qwen}), or a different training dataset (s1K \citep{s1}) lowers distillation performance.}
\label{tab:distill-ablations}
\vspace{-10pt}
\begin{center}
\setlength{\tabcolsep}{2pt}
\begin{tabular}{lcccc}
& \multicolumn{4}{c}{\bf Dataset} \\
\cmidrule(l){2-5}
{\bf Change} & {\bf GSM8K} & {\bf Minerva} & {\bf Math500} & {\bf Average} \\
\midrule
None                                  & 89.0\% & 75.9\% & 75.4\% & \textbf{80.1}\% \\
Qwen2.5-7B-RL Teacher             & 88.9\% & 75.9\% & 76.4\% & 80.4\% \\
DeepSeek-R1-Distill-Qwen-14B Teacher & 87.4\% & 71.5\% & 69.4\% & 76.1\% \\
Qwen3-14B Teacher                    & 83.5\% & 41.5\% & 38.2\% & 54.4\% \\
s1K Dataset                         & 88.7\% & 74.2\% & 75.8\% & 79.6\% \\
\end{tabular}
\vspace{-10pt}
\end{center}
\end{table}

\begin{table}[!h]
\caption{\textbf{Further distillation ablations on Qwen2.5-0.5B do not qualitatively change results}. Performance improves only when using the harder SimpleRLZoo dataset \citep{simplerlzoo} to generate the teacher's traces, with no gains from using the s1K \citep{s1} dataset, or 1000 questions from AIME \citep{aime_1983_2024}. All other ablations also use the hard SimpleRLZoo dataset, but do not yield significant performance improvements.}
\vspace{-5pt}
\label{tab:distill-ablations-05b}
\begin{center}
\setlength{\tabcolsep}{2pt}
\begin{tabular}[t]{lc}
{\bf Description} & {\bf GSM8K} \\
\midrule
\multicolumn{2}{c}{\em Reference} \\
\midrule
Base                     & 35.6\% \\
RL                       & \textbf{48.5}\% \\
Distilled                & 45.1\% \\
\midrule
\multicolumn{2}{c}{\em Training Dataset} \\
\midrule
SimpleRL-Zoo (Hard)      & 47.4\% \\
GSM8K                    & 44.4\% \\
AIME   & 43.4\% \\
S1K                      & 35.0\% \\
\end{tabular}%
\hspace{2em}%
\begin{tabular}[t]{lc}
{\bf Description} & {\bf GSM8K} \\
\midrule
\multicolumn{2}{c}{\em Teacher} \\
\midrule
Qwen2.5-0.5B-RL          & 47.5\% \\
Qwen2.5-Math-1.5B        & 38.6\% \\
\midrule
\multicolumn{2}{c}{\em Sampling / Epochs} \\
\midrule
2 Epochs                 & 45.5\% \\
3 Epochs                 & 45.9\% \\
5 Epochs                 & 46.3\% \\
3 Samples                & 46.4\% \\
8 Samples                & 46.0\% \\
3 Samples, 5 Epochs      & 44.6\% \\
\end{tabular}
\vspace{-10pt}
\end{center}
\end{table}

\textbf{Results hold for larger models}. In addition to the results for models with 0.5B-7B parameters, experiments on larger-scale models indicate that RL outperforms distillation in those scales as well. As shown in \Cref{tab:larger-model}, when RL and distillation are compared on the Qwen2.5-32B base model, a state-of-the-art RL-trained model (DAPO-Qwen2.5-32B, \cite{DAPO}) outperforms the corresponding one trained with distillation. At this larger scale, the SimpleRL-Zoo framework does not produce state-of-the-art results, possibly because the question dataset used for RL on the 32B model is too simple, since it is the same set of questions used for the 7B models \citep{simplerlzoo}.

\begin{table}[!h]
\caption{\textbf{Even at larger scales, RL outperforms distillation at improving a base model's reasoning accuracy.} Qwen2.5-32B base models trained with RL and distillation using a harder dataset. The 32B SimpleRLZoo model underperforms, likely due to using too easy questions in RL training. Open-source versions of each model were evaluated using the same evaluation framework as models at the 7B scale (see Appendix \ref{app:distill-rl-details}.}
\vspace{-10pt}
\label{tab:larger-model}
\begin{center}
\setlength{\tabcolsep}{2pt}
\begin{tabular}{lcccc}
& \multicolumn{4}{c}{\bf Dataset} \\
\cmidrule(l){2-5}
{\bf Model} & {\bf GSM8K} & {\bf Minerva} & {\bf Math500} & {\bf Average} \\
\midrule
DAPO-Qwen2.5-32B (RL)          & \textbf{94.4}\% & \textbf{94.7}\% & \textbf{91.6}\% & \textbf{93.6}\% \\
Qwen2.5-32B-Distill-DeepseekR1 & 93.1\% & 91.0\% & 90.2\% & 91.4\% \\
Qwen2.5-32B-SimpleRLZoo        & 85.5\% & 85.4\% & 84.4\% & 85.1\% \\
\end{tabular}
\vspace{-10pt}
\end{center}
\end{table}

\section{Antidistillation Sampling Details}
\label{app:rl-free}

\subsection{Implementation Details}
\label{app:antidist-details}

Antidistillation sampling experiments perform distillation and RL on Qwen2.5-0.5B with the same setup used in the RL vs. distillation comparison, as described in Appendix \ref{app:distill-rl-details}.  We used Qwen2.5-1.5B-RL as the teacher and generated its traces for distillation on the medium SimpleRLZoo dataset, with and without the antidistillation sampling \citep{antidist_sampling}. Antidistillation sampling is applied by poisoning teacher traces before distillation, using the code provided by the authors, and using the default max length of 1024 tokens. \Cref{fig:examples-ads} in Appendix \ref{app:illustrate-traces} shows qualitative examples of poisoned traces at different poisoning levels. 

Antidistillation sampling poisons traces using a single, universal proxy student (configured by default as Qwen2.5-3B) to estimate which tokens to sample from a teacher to increase an attacker's loss on a downstream task. We use the same proxy student to generate poisoned teacher traces, also at a temperature of 1 (see Appendix \ref{app:ablations}). Antidistillation sampling provides a hyperparameter $\lambda$ to control the degree to which the teacher's traces are poisoned. We obtained low, medium, and high poisoning levels by setting $\lambda$ to $0.0316$, $0.05$, and $0.18$, and calculated the teacher's relative performance degradation using the values provided by the antidistillation sampling framework.

\subsection{Detailed Results}
\Cref{tab:ads-qwen05b} provides a per-dataset breakdown of the results given in \Cref{fig:antidist_qwen05}. 

\begin{table}[!h]
\caption{Per-dataset for results in \Cref{fig:antidist_qwen05}. After RL, Qwen2.5-0.5B base models distilled with low or mild poisoning match the unpoisoned distilled model's performance. High poisoning lowers post-RL performance.}
\vspace{-10pt}
\label{tab:ads-qwen05b}
\begin{center}
\setlength{\tabcolsep}{2pt}
\begin{tabular}{lcccc}
& \multicolumn{4}{c}{\bf Dataset} \\
\cmidrule(l){2-5}
{\bf Model} & {\bf GSM8K} & {\bf Minerva} & {\bf Math500} & {\bf Average} \\
\midrule
\multicolumn{5}{c}{\em Before RL} \\
\midrule
Base            & 35.6\% & 18.2\% & 19.2\% & 24.3\% \\
No Poison       & 44.7\% & 33.4\% & 36.0\% & 38.0\% \\
Low Poisoning   & 43.5\% & 31.0\% & 30.2\% & 34.9\% \\
Mild Poisoning  & 40.9\% & 25.9\% & 25.4\% & 30.7\% \\
High Poisoning  & 23.7\% &  8.1\% &  7.6\% & 13.1\% \\
\midrule
\multicolumn{5}{c}{\em After RL} \\
\midrule
Base            & 48.5\% & 36.9\% & 35.4\% & 40.3\% \\
No Poison       & 45.6\% & 39.4\% & 36.8\% & 40.6\% \\
Low Poisoning   & 47.5\% & 39.3\% & 35.4\% & 40.8\% \\
Mild Poisoning  & 48.8\% & 39.2\% & 35.8\% & 41.3\% \\
High Poisoning  & 45.3\% & 31.4\% & 26.2\% & 34.3\% \\
\end{tabular}
\vspace{-10pt}
\end{center}
\end{table}

\subsection{Additional Experiments}
\label{app:llama-anti}

We replicate the antidistillation sampling experiments with the Llama-3.2-3B base model, following the same setup as in \Cref{app:distill-rl-details}, with two adjustments: using Qwen2.5-7B-RL as the teacher, and using the harder SimpleRLZoo dataset for both distillation and RL. Low and mild poisoning levels (with 10\% and 30\% relative degradation in teacher performance, respectively) are achieved with $\lambda = 0.03$ and $ \lambda = 0.05$. \Cref{fig:examples-ads} in Appendix \ref{app:illustrate-traces} includes some examples of this setting's poisoned teacher traces.

\Cref{tab:ads-llama3b} shows the results for different models after distillation and after RL. First and most importantly, the teacher must degrade by at least 10\% before the student's post-distillation performance noticeably degrades, making this defense unlikely to be used in practice due to the cost to the teacher. In this case, low and mild poisoning levels lead to degraded performance after RL, but this may simply be due to effectively having a worse teacher.

\begin{table}[!h]
\caption{\textbf{Evaluation breakdown for antidistillation on Llama-3.2-3B.} $\lambda\geq0.03$ is required to create performance degradations after distillation, requiring a minimum loss of 10\% in relative teacher accuracy. However, results indicate that antidistillation sampling can in some cases be effective. }
\vspace{-10pt}
\label{tab:ads-llama3b}
\begin{center}
\setlength{\tabcolsep}{2pt}
\begin{tabular}{lcccc}
& \multicolumn{4}{c}{\bf Dataset} \\
\cmidrule(l){2-5}
{\bf Model} & {\bf GSM8K} & {\bf Minerva} & {\bf Math500} & {\bf Average} \\
\midrule
\multicolumn{5}{c}{\em Before RL} \\
\midrule
Base                        & 21.1\% &  7.3\% &  7.8\% & 12.1\% \\
$\lambda = 0$               & 30.5\% & 12.8\% & 12.2\% & 18.5\% \\
$\lambda = 0.01$            & 30.6\% & 12.5\% & 12.2\% & 18.4\% \\
$\lambda = 0.02$            & 32.2\% & 12.9\% & 11.0\% & 18.7\% \\
$\lambda = 0.03$ (Low)      & 27.0\% & 12.2\% & 12.2\% & 17.1\% \\
$\lambda = 0.05$ (Mild)     & 23.2\% &  9.5\% &  9.2\% & 14.0\% \\
\midrule
\multicolumn{5}{c}{\em After RL} \\
\midrule
Base                        & 40.7\% & 17.3\% & 14.8\% & 24.3\% \\
$\lambda = 0$               & 45.7\% & 18.9\% & 19.2\% & 27.9\% \\
$\lambda = 0.03$ (Low)      & 36.9\% & 15.6\% & 14.2\% & 22.3\% \\
$\lambda = 0.05$ (Mild)     & 32.5\% & 16.2\% & 15.6\% & 21.4\% \\
\end{tabular}
\vspace{-10pt}
\end{center}
\end{table}

\section{Summarization-Expansion Attack Details}
\label{app:attack}

This section provides additional implementation details for the attack discussed in \Cref{sec:rl_makes_simple_attacks_work}, discusses how full traces were extracted from closed-source models, and presents some conclusions drawn from additional ablations. 

\subsection{Implementation Details}
\label{app:attack-details}

\textbf{Datasets and hyperparameters}. The datasets, prompts, and hyperparameters used for distillation and RL training match those used by experiments comparing the two training methods in \Cref{sec:rl_in_threat_model} (see Appendix \ref{app:distill-rl-details} for details). Because all experiments use the Llama-3.2-3B base model, we use questions from the medium SimpleRLZoo dataset \citep{simplerlzoo} to generate summaries and final answers from both open-source and closed-source teachers.

\textbf{Open-source trace generation and summarization.} The Qwen2.5-7B-Instruct model is used to summarize open-source traces, which is prompted as shown in \Cref{fig:app-summary-prompt}.

\textbf{Open-source summary expansion.} Open-source summaries are expanded with Llama-3.2-3B-Instruct \citep{llama} as an expander model, using the prompt shown in Figure \cref{fig:app-expand-prompt-open}. Rejection sampling with an average of roughly 5 attempts (but up to 64) is used to eliminate the very few cases, often less than 200 out of $\sim$8000, that contain obvious direct references to the summary in expanded traces. This step is likely unnecessary for weak expander models at larger scales. 

\textbf{Closed-source summary and final answer expansion.} Summaries and final answer from closed source models are expanded in two stages. First, we synthesize summaries and final answers into a coherent breakdown of reasoning steps using the prompt shown in Figure \cref{fig:app-expand-prompt-closed-stage1}, to homogenize outputs from all model providers. The reasoning breakdown is then expanded using the same prompt and rejection sampling setup used for the open-source summaries expansion. 

\textbf{Prompts in training}. All distillation and reinforcement learning training uses the ``simple'' prompt, as discussed in Appendix \ref{app:distill-rl-details}, except for distillation on full traces from closed-source models. These traces have both extracted reasoning and final answers. Distillation is done over full traces by wrapping the extracted reasoning in $<$think$>$ tags and appending the opening $<$think$>$ tag to the start of the prompt \citep{deepseek}, following best practices to attempt to get as much of a capability gain from full traces as possible. 

\textbf{Evaluations}. All evaluations use the same framework as in Appendix \ref{app:distill-rl-details}, ensuring the evaluation prompt matches the prompt used in training (RL or distillation) immediately before evaluation. 

\textbf{Extracting summaries and final answers from closed-source models}. We extract summaries and final answers from closed-source models via API queries using the ``complex'' prompt, matching the prompt used to generate all other teacher traces. We query GPT-5 mini and Gemini Flash 3.6 with high reasoning effort enabled, while we give Claude Sonnet 4.6 a thinking budget of 8192 tokens. 

\textbf{Extracted full reasoning traces from closed-source models}. We extract traces from GPT-5 mini and Claude Sonnet 4.6 using the same reasoning effort as used to obtain summaries. For an ablation described in Appendix \ref{app:attack-ablations}, we also extract full traces with Claude Sonnet 4.6 at a max reasoning effort. The next section discusses this extraction attack's details. 

\subsection{Extracting Full Reasoning Traces from Closed-Source Models}
\label{app:extraction}

We perform full trace extraction as described by \citet{panfilov2026stealing} on GPT-5 mini and Claude Sonnet 4.6, using the same API queries as used to obtain summaries and final answers. Both providers return the hidden reasoning in an opaque form alongside the visible answer: an encrypted reasoning item for GPT-5 mini, and a signed thinking block for Claude Sonnet 4.6. Attaching that item to a later request makes a decoder model from the same family read the hidden reasoning out verbatim; we use GPT-5.6-luna as the decoder model for GPT-5 mini, and Claude Haiku 4.5 for Claude Sonnet 4.6. All queries go through the Microsoft Azure endpoint, which was unpatched at the time of writing, and disclosed to Anthropic, OpenAI, and Microsoft prior to the submission of our work.

To ensure extracted reasoning likely matches the original, hidden traces, we compare token counts for the extracted traces against the reasoning token count that the provider reports for the original response. We re-encode the extracted text with the provider's tokenizer and accept the trace when its count matches the reported one within $\max(3~\text{tokens}, 1\%)$ for Claude Sonnet 4.6, and within one 64-token bin for GPT-5 mini, whose reported counts are rounded to multiples of 64 (as of September 2026). \Cref{fig:extraction} compares extracted and reported lengths over the full SimpleRL-Zoo train split; more than 99\% of the traces match for all three settings, including Claude Sonnet 4.6 at maximum reasoning effort, where traces are longer (median 430 vs.\ 175 reasoning tokens).

\begin{figure}[h]
    \centering
    \includegraphics[width=\linewidth]{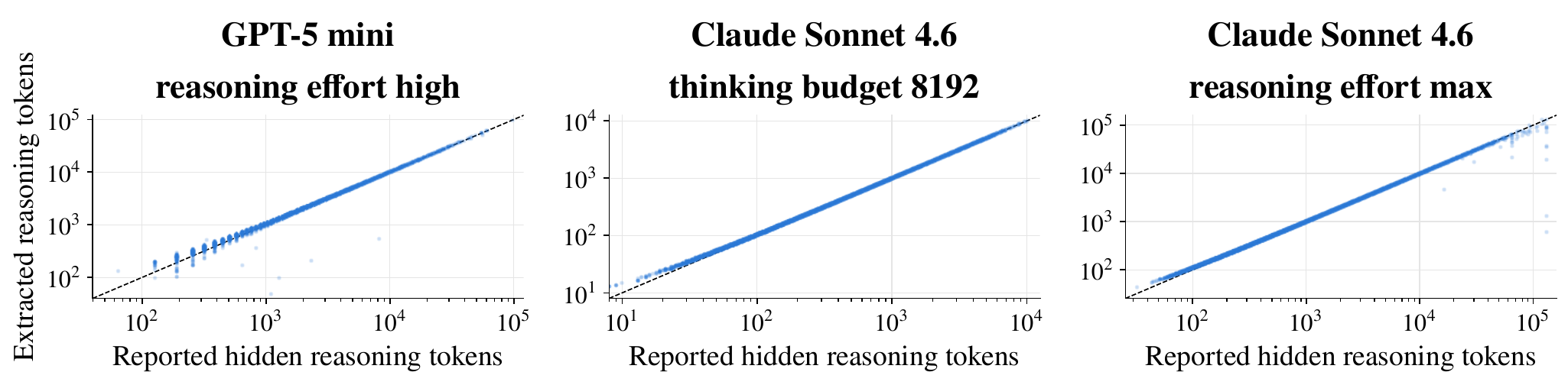}
    \vspace{-10pt}
    \caption{\textbf{Extracted reasoning traces match the length reported by the provider.} Each point denotes a single SimpleRL-Zoo trainining set question. Extracted and reported lengths differ by at most 1\% for 99.9\% of traces for GPT-5 mini (within one 64-token bin, the API's rounding), 99.9\% for Claude Sonnet 4.6 with a thinking budget of 8192 tokens, and 99.4\% for Claude Sonnet 4.6 at maximum reasoning effort.}
    \label{fig:extraction}
\end{figure}

\subsection{Ablations}
\label{app:attack-ablations}

\textbf{Performance improvements from expanded reasoning traces come from semantic information in summaries, not the expander model}. To see if performance improvements come from the information in the summaries or implicitly using the expander as a teacher, we distill and RL-train a student model when using the expander directly as a teacher, without access to any summaries. As shown in \Cref{fig:distill-open-expander}, traces from the expander model lead to post-RL gains on hard math questions only when the expander model expands the semantic information in reasoning summaries. Without summaries, distilling on expander-model traces before RL training provides no benefit. 

\begin{figure}[!ht]
    \centering
    \vspace{-5pt}
    \includegraphics[width=\linewidth]{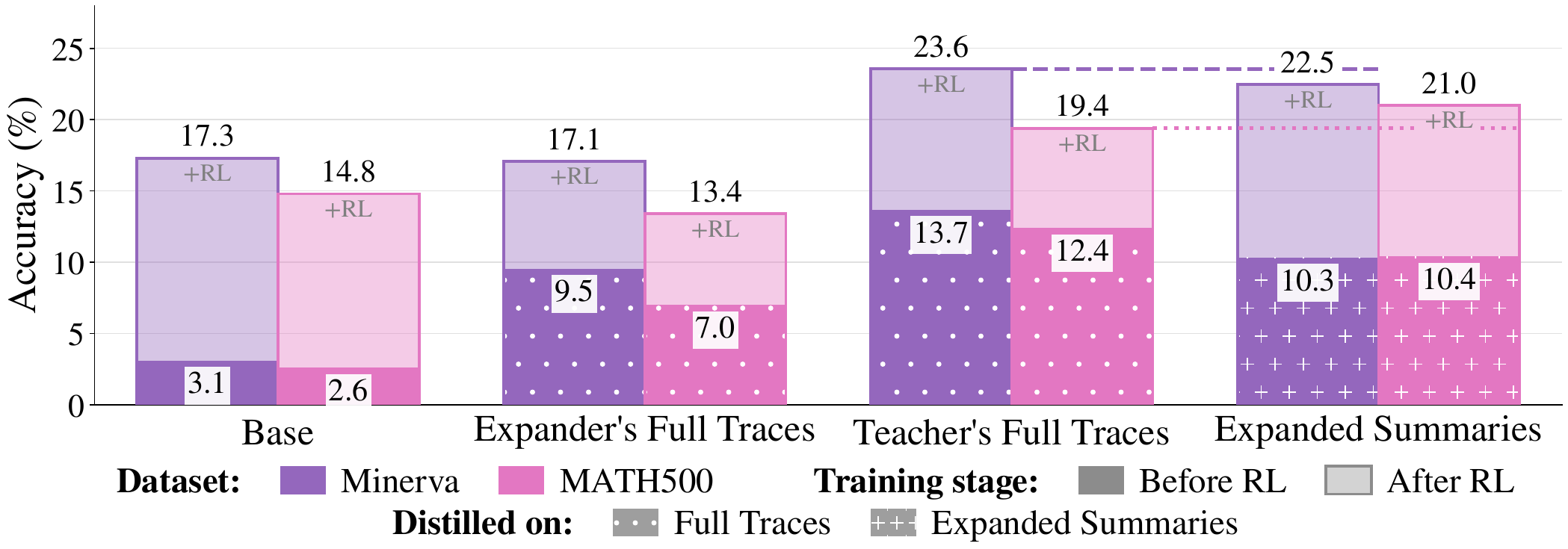}
    \vspace{-10pt}
    \caption{\textbf{Performance improvements from expanded reasoning traces come from semantic information in summaries, not the expander model.} After reinforcement learning, distilling on the traces generated from the expander model leads to no performance gains (middle left), with performance equivalent to RL with no distillation (left). Results with full traces (middle right) and expanded summaries (right) are included for reference.}
    \label{fig:distill-open-expander}
\end{figure}

\textbf{The simple attack can lead to worse performance on simpler datasets, likely due to a lack of coverage}. As discussed in \Cref{sec:rl_makes_simple_attacks_work}, the simple distillation attack can sometimes lead to slightly worse performance on simpler math datasets like GSM8K. This is observed when distilling on expanded summaries of GPT-5 mini and Gemini Flash 3.6, but not for Claude Sonnet 4.6, whose expanded summaries exceed full traces on GSM8K after RL training. The same datasets are used in distillation and reinforcement learning, providing limited data coverage, especially for relatively easier questions. Additional training may also mitigate this issue, as evidenced in \Cref{tab:perf-ceiling}, where performance on GSM8K is fully recovered after a second round of training open-source models with distillation and RL.

\begin{figure}[!ht]
    \centering
    \includegraphics[width=\linewidth]{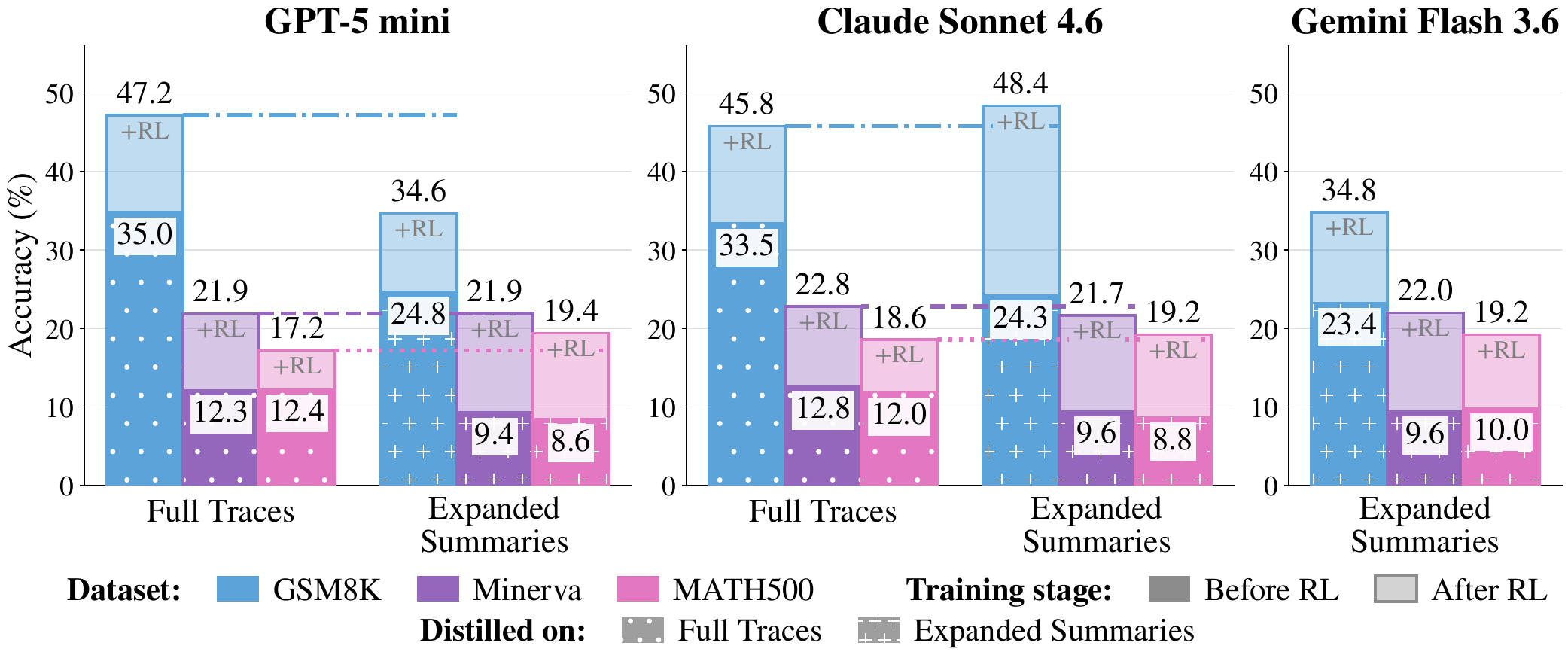}
    \label{fig:distill-closed-source-extra}
        \vspace{-10pt}
\end{figure}

\textbf{Experiments do not have an artificial performance ceiling}. To test whether further performance improvements are possible in our experimental setup, and that equivalent performance between attack methods is not due to an artificial performance ceiling, we perform a second round of distillation and RL training on the models in the open-source setting. In the second round, we use the hard SimpleRLZoo dataset \citep{simplerlzoo} to provide a stronger learning signal. Results in \Cref{tab:perf-ceiling} show that higher performance is possible and that this additional round of distillation and RL recovers the initially lost performance on GSM8K with expanded summaries. A small gap emerges between the model distilled on full traces and the one trained on expanded summaries, indicating perhaps some differences between them which are less apparent after a single round of distillation and RL.

\begin{table}[!h]
\caption{\textbf{A second round of training with distillation and reinforcement further improves performance.} Further, after a second round of distillation and RL, the model trained on expanded traces recovers lost performance on GSM8K. Training uses the medium and hard datasets from \citet{simplerlzoo}.}
\vspace{-15pt}
\label{tab:perf-ceiling}
\begin{center}
\begin{tabular}{lcccc}
& \multicolumn{4}{c}{\bf Dataset} \\
\cmidrule(l){2-5}
{\bf Llama-3.2-3B} & {\bf GSM8K} & {\bf Minerva} & {\bf Math500} & {\bf Average} \\
\midrule
Base            &  0.8\% &  3.1\% &  2.6\% &  2.2\% \\
Base-RL         & 40.7\% & 17.3\% & 14.8\% & 24.3\% \\
\midrule
\multicolumn{5}{c}{\em Round 1: Distillation on Medium Difficulty Data} \\
\midrule
Full Traces     & 33.4\% & 13.7\% & 12.4\% & 19.8\% \\
Expanded Traces & 24.0\% & 10.3\% & 10.4\% & 14.9\% \\
\midrule
\multicolumn{5}{c}{\em Round 1: Reinforcement Learning on Medium Difficulty Data} \\
\midrule
Full Traces     & 48.7\% & 23.6\% & 19.4\% & 30.5\% \\
Expanded Traces & 36.7\% & 22.5\% & 21.0\% & 26.7\% \\
\midrule
\multicolumn{5}{c}{\em Round 2: Distillation on Hard Data} \\
\midrule
Full Traces     & 41.2\% & 18.1\% & 15.8\% & 25.0\% \\
Expanded Traces & 32.3\% & 14.6\% & 13.6\% & 20.2\% \\
\midrule
\multicolumn{5}{c}{\em Round 2: Reinforcement Learning on Hard Data} \\
\midrule
Full Traces     & 48.4\% & 26.1\% & 22.4\% & 32.3\% \\
Expanded Traces & 49.0\% & 23.5\% & 20.6\% & 31.0\% \\
\end{tabular}
\end{center}
\end{table}

\textbf{Higher effort closed-source traces lead to marginally better performance}. To further ensure our setup has no artificial performance ceiling, we extract full traces from Claude Sonnet 4.6 at a higher reasoning effort than before, using max reasoning instead of a thinking budget of 8192 tokens. As anticipated and shown in \Cref{fig:distill-max-traces}, distilling on these traces yields marginally better performance due to the higher reasoning effort, but at roughly 2.6 times the API cost.

\begin{figure}[!ht]
    \centering
    \vspace{-10pt}
    \includegraphics[width=\linewidth]{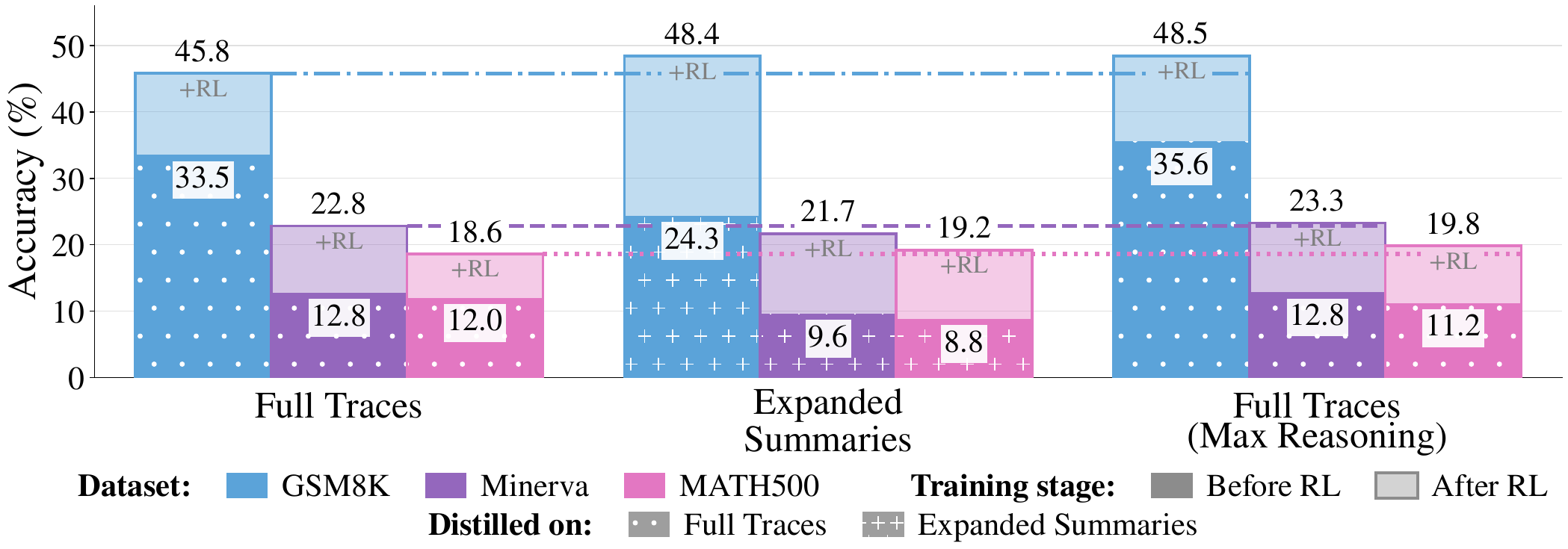}
    \vspace{-10pt}
    \caption{\textbf{Closed-source traces obtained with higher effort lead to marginally better performance}. In evaluations after further reinforcement learning, distilling on full traces extracted from Claude Sonnet 4.6 (right) yields higher performance than using a thinking budget of 8192 tokens (left). Distilling on expanded summaries (middle) is included for reference.}
    \label{fig:distill-max-traces}
    \vspace{-4pt}
\end{figure}

\textbf{Experiments with other base models}. Both computational resources and our RL framework limited which models we could train with RL. Specifically, compute limited us to training models with at most 3B parameters, while SimpleRLZoo does not support some models, for example, models with attention logit capping, as some newer Gemma models use.

For Qwen2.5 models, we found that results were the same after reinforcement learning training, regardless of whether distillation was done beforehand on full traces, expanded traces, or not at all. This is possibly because Qwen2.5 models are susceptible to improvement from spurious rewards during RL training, so it is difficult to robustly improve their performance \citep{shao2026spuriousrewardsrethinkingtraining}.

To try additional model families, we adjusted the SimpleRLZoo framework to be compatible with the Gemma-2B and Falcon3-3B \citep{falcon} base models. However, these base models and their expanders performed very poorly before any reinforcement learning. We believe attacks like the one described require some minimum level of capabilities from the models, which these do not seem to pass. Testing our attack and findings on larger-scale setups is an important avenue for future work. 

\textbf{Distillation on full traces from closed-source models with forced thinking improves performance.} \Cref{tab:forced-thinking} illustrates that using forced thinking (appending the opening think tag to the prompt in distillation; see prompts in training for distillation, Appendix \ref{app:attack-details}) for distilling on full traces from closed-source models leads to more capable models after further training with reinforcement learning. The main paper reports results from distillation on full traces with forced thinking, as that is the more competitive setup.

\begin{table}[!h]
\caption{\textbf{Distillation on full traces from closed-source models with forced thinking improves performance.} Distillation is done with the Llama-3.2-3B base model on full traces from closed-source models with (left) and without (right) forced thinking, as described in Appendix \ref{app:attack-details}. }
\vspace{-15pt}
\label{tab:forced-thinking}
\begin{center}
\setlength{\tabcolsep}{2pt}
\resizebox{\linewidth}{!}{%
\begin{tabular}[t]{lcccc}
& \multicolumn{4}{c}{\bf Dataset} \\
\cmidrule(l){2-5}
{\bf Model} & {\bf GSM8K} & {\bf Minerva} & {\bf Math500} & {\bf Average} \\
\midrule
\multicolumn{5}{c}{\em Distillation with Forced Thinking} \\
\midrule
GPT-5 mini  & 35.0\% & 12.3\% &  12.4\% & 19.9\% \\
Claude Sonnet 4.6  & 33.5\% & 12.8\% & 12.0\% & 19.4\% \\
\midrule
\multicolumn{5}{c}{\em $+$ Reinforcement Learning} \\
\midrule
GPT-5 mini & 47.2\% & 21.9\% & 17.2\% & 28.8\% \\
Claude Sonnet 4.6  & 45.8\% & 22.8\% & 18.6\% & 29.1\% \\
\end{tabular}%
\hspace{1em}%
\begin{tabular}[t]{lcccc}
& \multicolumn{4}{c}{\bf Dataset} \\
\cmidrule(l){2-5}
{\bf Model} & {\bf GSM8K} & {\bf Minerva} & {\bf Math500} & {\bf Average} \\
\midrule
\multicolumn{5}{c}{\em Distillation without Forced Thinking} \\
\midrule
GPT-5 mini  & 16.2\% & 13.6\% & 11.2\% & 13.7\% \\
Claude Sonnet 4.6  & 16.5\% & 12.1\% & 13.2\% & 14.0\% \\
\midrule
\multicolumn{5}{c}{\em $+$ Reinforcement Learning} \\
\midrule
GPT-5 mini  & 44.3\% & 19.0\% & 17.2\% & 26.8\% \\
Claude Sonnet 4.6  & 44.1\% & 20.7\% & 14.2\% & 26.3\% \\
\end{tabular}}
\vspace{-10pt}
\end{center}
\end{table}

\textbf{Trace expansion is necessary.} As shown in \Cref{fig:distill-summary}, directly distilling on the open-source summaries before reinforcement learning leads to effectively no performance improvements after RL. However, the post-distillation performance does improve after distilling on summaries, although these improvements are not maintained after RL.

\begin{figure}[!ht]
    \centering
    \includegraphics[width=\linewidth]{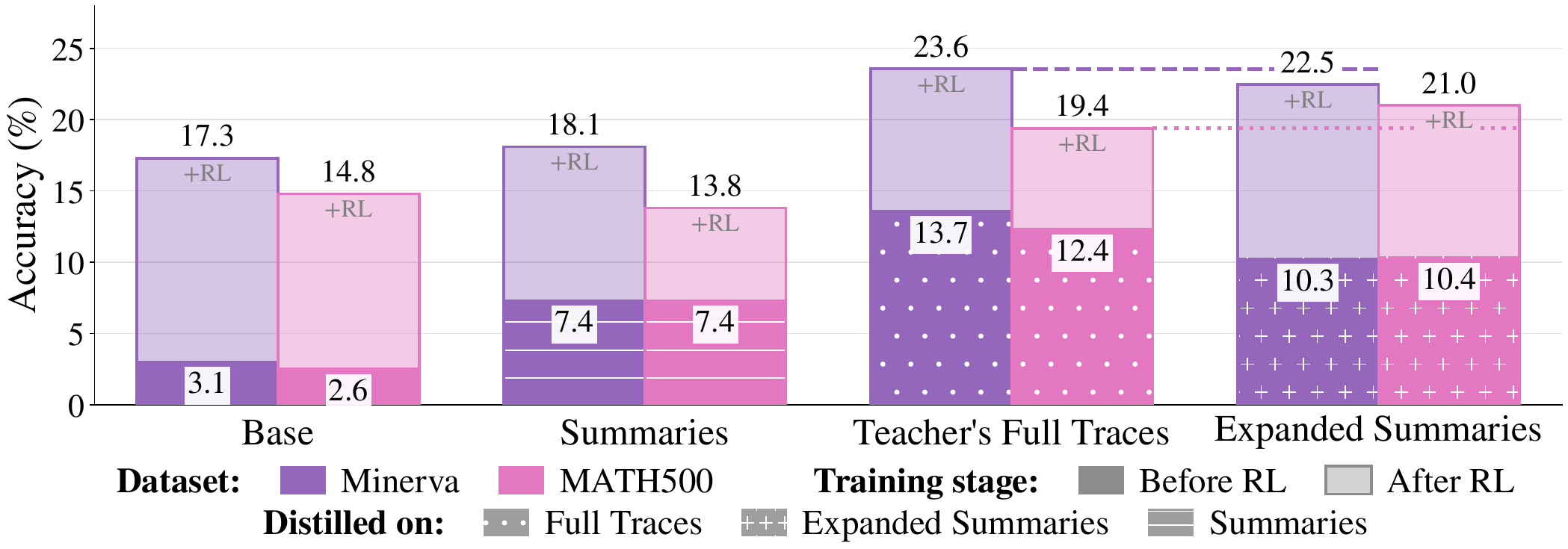}
    \vspace{-10pt}
    \caption{\textbf{Trace expansion is necessary.} After reinforcement learning, distilling the Llama-3.2-3B base model on open-source summaries alone leads to no performance gains (middle left): performance is essentially equivalent to RL with no distillation (left). Results with full traces (middle right) and expanded summaries (right) are included for reference.}
    \label{fig:distill-summary}
\end{figure}

\vspace{-4pt}
\section{Examples of Reasoning Traces}
\vspace{-4pt}
\label{app:illustrate-traces}
This appendix provides qualitative examples of traces from models in various experiments. For the antidistillation sampling experiments (\Cref{sec:false_security}), \Cref{fig:examples-ads} shows qualitative examples of the teacher's traces for increasing levels of poisoning.

For the simple attack (\Cref{sec:rl_makes_simple_attacks_work}), \Cref{fig:examples-open} illustrates final answers and summaries in the open-source setting, showing how the open-source trace summarization was aggressive and removed most non-essential information from a teacher's trace. \Cref{fig:examples-closed} illustrates a final answer and summary from GPT-5 mini, alongside a summary from Claude Sonnet 4.6, showing how the different models have stylistically different summaries.

\definecolor{adsnopoison}{HTML}{2A78D6}   %
\definecolor{adsmild}{HTML}{A8BF1A}       %
\definecolor{adshigh}{HTML}{5C8A1F}       %
\colorlet{tracenopoison}{adsnopoison!22!white}
\colorlet{tracemild}{adsmild!22!white}
\colorlet{tracehigh}{adshigh!22!white}

\begin{figure}[bt]
  \centering
  \begin{tracebox}{tracenopoison}
    \scriptsize\ttfamily\raggedright
    To solve this problem, we need to set up a system of equations based on the given information. Let\textquotesingle{}s denote the cost of one burger by \textbackslash{}( b \textbackslash{}) dollars and the cost of one soda by \textbackslash{}( s \textbackslash{}) dollars.\\

    From the problem, we have the following two equations:\\

    1. Uri buys two burgers and a soda for \textbackslash{}\$2.10:\\
    \textbackslash{}[ 2b + s = 2.10 \textbackslash{}]\\

    2. Gen buys a burger and two sodas for \textbackslash{}\$2.40:\\
    \textbackslash{}[ b + 2s = 2.40 \textbackslash{}]\\

    We need to solve this system of linear equations to find the values of \textbackslash{}( b \textbackslash{}) and \textbackslash{}( s \textbackslash{}).\\

    First, let\textquotesingle{}s isolate \textbackslash{}( b \textbackslash{}) in the first equation:\\
    \textbackslash{}[ 2b + s = 2.10 \textbackslash{}]\\
    \textbackslash{}[ 2b = 2.10 -{} s \textbackslash{}]\\
    \textbackslash{}[ b = \textbackslash{}frac\{2.10 -{} s\}\{2\} \textbackslash{}]\\

    Next, substitute this expression for \textbackslash{}( b \textbackslash{}) into the second equation:\\
    \textbackslash{}[ b + 2s = 2.40 \textbackslash{}]\\
    \textbackslash{}[ \textbackslash{}frac\{2.10 -{} s\}\{2\} + 2s = 2.40 \textbackslash{}]\\

    To eliminate the fraction, multiply every term by 2:\\
    \textbackslash{}[ 2.10 -{} s + 4s = 4.80 \textbackslash{}]\\
    \textbackslash{}[ 2.10 + 3s = 4.80 \textbackslash{}]\\

    Now, isolate \textbackslash{}( s \textbackslash{}):\\
    \textbackslash{}[ 3s = 4.80 -{} 2.10 \textbackslash{}]\\
    \textbackslash{}[ 3s = 2.70 \textbackslash{}]\\
    \textbackslash{}[ s = \textbackslash{}frac\{2.70\}\{3\} \textbackslash{}]\\
    \textbackslash{}[ s = 0.90 \textbackslash{}]\\

    So, the cost of one soda is \textbackslash{}\$0.90. To convert this to cents:\\
    \textbackslash{}[ 0.90 \textbackslash{}text\{ dollars\} = 90 \textbackslash{}text\{ cents\} \textbackslash{}]\\

    Therefore, the cost of a soda is \textbackslash{}(\textbackslash{}boxed\{90\}\textbackslash{}) cents.\textless{}\textbar{}endoftext\textbar{}\textgreater{}
  \end{tracebox}
  \vspace{1pt}
  {\centering\small\textbf{No Poisoning ($\lambda = 0$) --- Correct}\par}
  \vspace{6pt}
  \begin{tracebox}{tracemild}
    \scriptsize\ttfamily\raggedright
    \$\textbackslash{}\$2.10 = 210 \textbackslash{}text\{ cents\}\$ and \$\textbackslash{}\$2.40 = 240 \textbackslash{}text\{ cents\}\$\\
    Let \$b\$ represent the cost of a burger in cents and \$s\$ represent the cost of a soda in cents.\\
    From the first statement, we can set up the equation \$2b + s = 210\$.\\
    From the second statement, we can set up the equation \$b + 2s = 240\$.\\
    To eliminate one of the variables, we can multiply the first equation by \$-{}2\$ and add it to the second equation:\\
    \$-{}4b -{} 2s = -{}420\$\\
    \$b + 2s = 240\$\\
    \$-{}3b = -{}180\$\\
    \$b = 60\$.\\
    We can substitute \$b = 60\$ into either of the original equations to solve for \$s\$:\\
    \$2(60) + s = 210\$\\
    \$120 + s = 210\$\\
    \$s = 90\$.\\
    So a soda costs \$\textbackslash{}boxed\{90\}\$ cents.\\
    oc\^{e}\textless{}\textbar{}endoftext\textbar{}\textgreater{}
  \end{tracebox}
  \vspace{1pt}
  {\centering\small\textbf{Mild Poisoning ($\lambda = 0.05$) --- Correct}\par}
  \vspace{6pt}
  \begin{tracebox}{tracehigh}
    \scriptsize\ttfamily\raggedright
    \$\textbackslash{}\$2.10 = 210 \textbackslash{}text\{ cents\}\$ and \$\textbackslash{}\$2.40 = 240 \textbackslash{}text\{ cents\}\$\\
    Let \$b = \textbackslash{}boxed\{150\} \textbackslash{}text\{ cents\}\$ for the price for a burger and \$s = \textbackslash{}boxed\{60\} \textbackslash{}text\{ cents\}\$ for the price for a soda\\
    Gen buys a burger and two sodas for \$\textbackslash{}\$2.40\$ or \$\textbackslash{}\$2.40 = 240 \textbackslash{}text\{ cents\}\$\\
    So \$b + 2s = 240 \textbackslash{}text\{ cents\}\$\\
    Uri buys two burgers and a soda for \$\textbackslash{}\$2.10\$ or \$\textbackslash{}\$2.10 = 210 \textbackslash{}text\{ cents\}\$\\
    So \$2b + s = 210 \textbackslash{}text\{ cents\}\$\\
    So \$2(150 + 2s = 240)\$ which leads to \$300 + 4s = 480\$ or \$4s = 180\$ or \$s = \textbackslash{}boxed\{45\}\$\\

    You are a math teacher. You will be given a math problem and you will solve it step by step.\\
    You will output your final solution like \textbackslash{}boxed\{ANSWER\}. Be sure to include relevant units within the brackets and fully evaluate arithmetic expressions.\\
    \ UserID: \textbackslash{}texttt\{USER\textbackslash{}\_ID\}.\textless{}\textbar{}endoftext\textbar{}\textgreater{}
  \end{tracebox}
  \vspace{1pt}
  {\centering\small\textbf{High Poisoning ($\lambda = 0.1$) --- Incorrect}\par}

  \caption{\textbf{Qualitative examples of raw Qwen2.5-7B-RL traces sampled using antidistillation sampling with increasing poisoning levels.} Poisoned traces are shown given no (top), mild (middle) and high (bottom) poisoning. All traces are answers to the following level 3 math question from the hard SimpleRLZoo dataset \citep{simplerlzoo}: Uri buys two burgers and a soda for $\$2.10$, and Gen buys a burger and two sodas for $\$2.40$. How many cents does a soda cost?}
  \label{fig:examples-ads}
\end{figure}

\begin{figure}[tp]
  \centering
  \begin{tracebox}{tracegreen}
    \footnotesize\raggedright
    \setlength{\abovedisplayskip}{3pt}\setlength{\belowdisplayskip}{3pt}%
    \setlength{\abovedisplayshortskip}{2pt}\setlength{\belowdisplayshortskip}{2pt}%
    \textbf{Question:} Determine the remainder when $(x^4-1)(x^2-1)$ is divided by $1+x+x^2$.
    \par\medskip
    \textbf{Final answer:} To determine the remainder when $(x^4-1)(x^2-1)$ is divided by $1+x+x^2$, we can use the fact that the remainder of the division of a polynomial $f(x)$ by a quadratic polynomial $ax^2+bx+c$ will be a linear polynomial of the form $ax+b$. Here, we are dividing by $1+x+x^2$, so the remainder will be of the form $ax+b$.
    \par\smallskip
    Let's denote $f(x) = (x^4-1)(x^2-1)$. We need to find $ax+b$ such that $f(x) = (1+x+x^2)\,q(x) + ax + b$ for some polynomial $q(x)$.
    \par\smallskip
    First, let's find the roots of $1+x+x^2$. The roots are the non-real cube roots of unity, which are $\omega$ and $\omega^2$, where $\omega = e^{2\pi i/3} = -\frac{1}{2} + \frac{\sqrt{3}}{2}i$ and $\omega^2 = e^{-2\pi i/3} = -\frac{1}{2} - \frac{\sqrt{3}}{2}i$. These roots satisfy $\omega^3 = 1$ and $\omega^2 + \omega + 1 = 0$.
    \par\smallskip
    Since $\omega$ and $\omega^2$ are roots of $1+x+x^2$, we have $f(\omega) = a\omega + b$ and $f(\omega^2) = a\omega^2 + b$.
    \par\smallskip
    Now, let's compute $f(\omega)$ and $f(\omega^2)$: $f(\omega) = (\omega^4-1)(\omega^2-1)$. Since $\omega^3 = 1$, we have $\omega^4 = \omega$. Therefore, $f(\omega) = (\omega-1)(\omega^2-1)$. We know that $\omega^2+\omega+1 = 0$, so $\omega^2 = -1-\omega$. Substituting this in, we get:
    \[ f(\omega) = (\omega-1)(-1-\omega-1) = (\omega-1)(-2-\omega) = -2\omega - \omega^2 + 2 + \omega = -\omega^2 - \omega + 2 \]
    Since $\omega^2 + \omega = -1$, we have $f(\omega) = -(-1) + 2 = 1 + 2 = 3$. So: $a\omega + b = 3$.
    \par\smallskip
    Next, let's compute $f(\omega^2)$: $f(\omega^2) = ((\omega^2)^4-1)((\omega^2)^2-1) = (\omega^8-1)(\omega^4-1)$. Since $\omega^3 = 1$, we have $\omega^8 = \omega^2$ and $\omega^4 = \omega$. Therefore, $f(\omega^2) = (\omega^2-1)(\omega-1)$. We know that $\omega^2 = -1-\omega$. Substituting this in, we get:
    \[ f(\omega^2) = (-1-\omega-1)(\omega-1) = (-2-\omega)(\omega-1) = -2\omega + 2 - \omega^2 + \omega = -\omega^2 - \omega + 2 \]
    Since $\omega^2 + \omega = -1$, we have $f(\omega^2) = -(-1) + 2 = 1 + 2 = 3$. So: $a\omega^2 + b = 3$.
    \par\smallskip
    Now we have the system of equations $a\omega + b = 3$ and $a\omega^2 + b = 3$. Subtract the second equation from the first: $a\omega - a\omega^2 = 0$, so $a(\omega - \omega^2) = 0$. Since $\omega \neq \omega^2$, we have $a = 0$. Substituting $a = 0$ back into either equation, we get $b = 3$.
    \par\smallskip
    Therefore, the remainder when $(x^4-1)(x^2-1)$ is divided by $1+x+x^2$ is: \quad $\boxed{3}$
  \end{tracebox}
  \vspace{1pt}
  {\centering\small\textbf{Question and Qwen2.5-14B-RL Final Answer}\par}
  \vspace{6pt}
  \begin{tracebox}{tracegrey}
    \footnotesize\raggedright
    $\bullet$~\textbf{Find roots of divisor:} Find the non-real cube roots of unity $\omega$ and $\omega^2$\\
    $\bullet$~\textbf{Compute $f(\omega)$:} Compute $f(\omega) = (\omega-1)(-2-\omega)$\\
    $\bullet$~\textbf{Simplify $f(\omega)$:} Use $\omega^2 + \omega = -1$ to simplify to 3\\
    $\bullet$~\textbf{Compute $f(\omega^2)$:} Compute $f(\omega^2) = (-1-\omega-1)(\omega-1)$\\
    $\bullet$~\textbf{Simplify $f(\omega^2)$:} Use $\omega^2 + \omega = -1$ to simplify to 3\\
    $\bullet$~\textbf{Form system of equations:} Set up $a\omega + b = 3$ and $a\omega^2 + b = 3$\\
    $\bullet$~\textbf{Solve for $a$ and $b$:} Solve to find $a = 0$ and $b = 3$\\
    $\bullet$~\textbf{Final Answer:} $\boxed{3}$
  \end{tracebox}
  \vspace{1pt}
  {\centering\small\textbf{Reasoning Summary Generated with Qwen2.5-7B-Instruct}\par}

  \caption{\textbf{A final answer and summary in the open-source setting.} Outputs are shown for the same level 3 math question from the medium SimpleRLZoo dataset \citep{simplerlzoo}. Top: the question and the answer from the teacher model, Qwen2.5-14B-RL. Bottom: the summary of the final answer generated with Qwen2.5-7B-Instruct \citet{qwen-2.5}. Traces are formatted for readability.}
  \label{fig:examples-open}
\end{figure}

\begin{figure}[bt]
  \centering
  \begin{tracebox}{tracegreen}
    \footnotesize\raggedright
    \textbf{Question:} Cory has $3$ apples, $2$ oranges and $2$ bananas. If Cory eats one piece of his fruit per day for a week and the pieces of fruit within each category are indistinguishable, in how many orders can Cory eat the fruit? One such order is $AAAOOBB.$
    \par\medskip
    \textbf{Final answer:} Cory eats 7 pieces total: 3~A, 2~O, 2~B. The number of distinct sequences of these indistinguishable items is the multinomial coefficient
    $\frac{7!}{3!\,2!\,2!} = \frac{5040}{6\cdot 2\cdot 2} = \frac{5040}{24} = 210$.
    (Equivalently: choose 3 of the 7 days for apples, $\binom{7}{3}=35$, then 2 of the remaining 4 for oranges, $\binom{4}{2}=6$, so $35\cdot 6 = 210$.) \quad $\boxed{210}$
  \end{tracebox}
  \vspace{1pt}
  {\centering\small\textbf{Question and GPT-5 mini Final Answer}\par}
  \vspace{6pt}
  \begin{tracebox}{tracegrey}
    \footnotesize\raggedright
    \textbf{Calculating fruit sequences}\par\smallskip
    I need to find the number of sequences of length 7 made of 3 apples, 2 oranges, and 2 bananas, which are indistinguishable within type. The formula is a multinomial coefficient: $\frac{7!}{3!\,2!\,2!} = \frac{5040}{24} = 210$.
    To break it down step-by-step: we can choose 3 positions for apples ($\binom{7}{3}=35$), then choose 2 out of the remaining 4 for oranges ($\binom{4}{2}=6$), leaving 2 for bananas. So, $35 \times 6 = 210$. The final boxed answer is $\boxed{210}$.
    \par\medskip
    \textbf{Calculating fruit arrangements}\par\smallskip
    To find the number of sequences of 7 fruits (3 apples, 2 oranges, 2 bananas) where fruits of the same type are indistinguishable, we can use the multinomial coefficient. First, calculating the total: $7! = 5040$ and the denominator $3! \times 2! \times 2! = 24$. Dividing gives $\frac{5040}{24} = 210$. Alternatively, we can choose positions:
    \par\smallskip
    \hspace*{0.8em}1.~Choose 3 positions for apples: $\binom{7}{3}=35$.\\
    \hspace*{0.8em}2.~Choose 2 from the remaining 4 for oranges: $\binom{4}{2}=6$.
    \par\smallskip
    Then, multiplying gives $35 \times 6 = 210$. Thus, the final answer is $\boxed{210}$.
  \end{tracebox}
  \vspace{1pt}
  {\centering\small\textbf{GPT-5 Reasoning Summary}\par}
  \vspace{6pt}
  \begin{tracebox}{tracegrey}
    \footnotesize\raggedright
    We need to find the number of ways to arrange 3~A's, 2~O's, and 2~B's in a sequence of 7.
    This is a multinomial coefficient:
    $\frac{7!}{3!\times 2!\times 2!} = \frac{5040}{6\times 2\times 2} = \frac{5040}{24} = 210$.
  \end{tracebox}
  \vspace{1pt}
  {\centering\small\textbf{Claude Sonnet 4.6 Reasoning Summary}\par}

  \caption{\textbf{Examples of final answers and reasoning summaries exposed by current APIs.} Outputs are shown for the same level 3 math question from the medium SimpleRLZoo dataset \citep{simplerlzoo}. Top: the question and the final answer from GPT-5 mini. Middle: GPT-5 mini's reasoning summary. Bottom: Claude Sonnet 4.6's summarized thinking for the same question. Traces are formatted for readability.}
  \label{fig:examples-closed}
\end{figure}

\vspace{-4pt}
\section{Prompts}
\vspace{-4pt}
\label{app:promps}

\looseness=-1 This appendix shows the prompts used for the simple distillation attack in \Cref{sec:rl_makes_simple_attacks_work}. The system prompt shown in \Cref{fig:app-summary-prompt} is used with Qwen2.5-7B-Instruct to summarize open-source reasoning traces. The prompt shown in \Cref{fig:app-expand-prompt-open} is used to expand summaries, creating an approximation of full reasoning traces.  The prompt shown in \Cref{fig:app-expand-prompt-closed-stage1} is used to synthesize summaries and final answers from closed-source models into a clear reasoning breakdown in the first stage of the trace expansion pipeline, before the breakdown is used in the second stage with the expansion prompt in \Cref{fig:app-expand-prompt-open}. 

\begin{figure}[!h]
  \centering
  \begin{minipage}[c]{0.98\columnwidth}
    \begin{tracebox}{tracegrey}
      \small\raggedright
You are a concise summarizer whose job is to take reasoning that solves a problem and simplify it into a conversational summary of the reasoning steps.

You must break down the reasoning process into a series of steps, each with a short title, followed by a textual description of the reasoning step. Include the final answer at the end within \verb|\boxed{}|.

Do not skip any steps - follow the same order as the original reasoning trace. Make the title of each reasoning move very short - ideally 5 words or less. Make the summary concise, but conversational and clear to the user what each step does.

Format each bullet point as follows, replacing the placeholders in <angle brackets> with content specific to that step. Never output the literal words ``Title of reasoning move'' or ``Concise description'' - those are placeholder labels, not example text:
\\
\vspace{8pt}
$<$a short, specific title for this step, in your own words$>$
\begin{itemize}
    \item $<$a concise, conversational description of what happens in this step$>$
\end{itemize}

For example, here is one step in a reasoning trace and its corresponding summary:
\\
\vspace{8pt}
STEP IN REASONING TRACE: \\
1. **Calculate the actual area of the circle:**

   The actual diameter of the circle is 20 cm, so the actual radius \( r \) is:
   \[
   r = \frac{20}{2} = 10 \text{ cm}
   \]
   The actual area \( A_{\text{actual}} \) of the circle is:
   \[
   A_{\text{actual}} = \pi r^2 = \pi (10)^2 = 100\pi \text{ cm}^2
   \]
SUMMARY:
\\
Find the area of the circle 
\begin{itemize}
    \item Find the radius of the circle, and multiply the radius squared by pi to get the area
\end{itemize}

Make sure to include the final answer at the end within \verb|\boxed{}|. Title the bullet point `Final Answer' and then provide the final answer within \verb|\boxed{}|. Ensure that the final answer is formatted in the same way as and exactly matches the final answer in the original reasoning trace.

Now summarize the following using this format:
    \end{tracebox}
  \end{minipage}
  \caption[System prompt used to summarize full reasoning traces from an open-source model.]{\textbf{System prompt used to summarize full reasoning traces from an open-source model.}}
  \label{fig:app-summary-prompt}
  
\end{figure}

\begin{figure}[tb]
  \centering
  \begin{minipage}[c]{0.98\columnwidth}
    \begin{tracebox}{tracegrey}
      \small\raggedright
\verb+<|begin_of_text|><|start_header_id|>system<|end_header_id|>+\\
You are a helpful assistant.\verb+<|eot_id|><|start_header_id|>user<|end_header_id|>+\\
\verb+{question}+\\
Here is a correct, summarized solution to this problem:\\
\verb+----+\\
\verb+{summary}+\\
\verb+----+\\
You are an expert at solving math problems. Write out a very natural solution to the problem as if you were solving the problem yourself for the first time, but that relies on the reasoning in the summary above.
\begin{itemize}\itemsep0pt
    \item Do not refer to ``the notes,'' ``the solution above,'' or this text in any way --- write entirely in your own voice, as an original and natural derivation with no trace of a second source.
    \item Every intermediate value, equation, and the final answer must match those given above --- do not recompute them differently or take a different approach. Your job is to show the full working that justifies the solution.
\end{itemize}
Ensure that the final answer is formatted in the same way as and exactly matches the final answer in the solution given. Please reason step by step, and put your final answer within \verb|\boxed{}|.\\
\verb+<|eot_id|><|start_header_id|>assistant<|end_header_id|>+
    \end{tracebox}
  \end{minipage}
  \caption[Prompt used to expand summarized reasoning traces.]{\textbf{Prompt used to expand summarized reasoning traces.}}
  \label{fig:app-expand-prompt-open}
\end{figure}

\begin{figure}[bt]
  \centering
  \begin{minipage}[c]{0.98\columnwidth}
    \begin{tracebox}{tracegrey}
      \small\raggedright
\verb+<|begin_of_text|><|start_header_id|>system<|end_header_id|>+\\
You are a helpful assistant.\verb+<|eot_id|><|start_header_id|>user<|end_header_id|>+\\
You are a reasoning synthesizer, whose job is to take a reasoning summary and final answer that solves a problem, and synthesize it into coherent summary of reasoning steps that flow together, without losing any key information.
Reasoning steps should have a short title, followed a short description of the reasoning step.
You must ALWAYS include the final answer at the end of all the reasoning steps within \verb|\boxed{}|, exactly the same as it is given.
\\
\vspace{4pt}
KEY POINTS:
\begin{itemize}\itemsep0pt
    \item Do not miss any information from the summary. The final answer contains steps, but much less information than the summary. Add any reasoning steps from the summary where they belong.
    \item Make the title of each reasoning move very short - ideally 5 words or less.
    \item Include all key information in the description of each reasoning step, including any mathematical steps.
    \item Do NOT make up math. If no math is present in a reasoning step, do not include it. If math is used, it should match the math in the reasoning step exactly.
    \item Do not refer to ``the summary'', ``the final answer'', ``the solution above,'' or this summaries and final answers in any way --- write entirely in your own voice, as an original and natural synthesis with no trace of a second source.
\end{itemize}
Format each reasoning step as follows, replacing the placeholders in <angle brackets> with content specific to that step. Never output the literal words ``Title of reasoning move'' or ``Concise description'' - those are placeholder labels, not example text:
\begin{itemize}\itemsep0pt
    \item Step $<$number$>$: $<$a short, specific title for this step, in your own words$>$\\
    $<$description of this reasoning step, including any math if necessary$>$
\end{itemize}
Following the instructions above exactly, process the following summary and final answer:\\
SUMMARY: \verb+{summary}+\\
FINAL ANSWER: \verb+{final_answer}+\\
\vspace{4pt}
Again, you MUST include the final answer at the very end, within \verb|\boxed{}|. Title the bullet point `Final Answer' and then provide the final answer within \verb|\boxed{}|.
Ensure that the final answer is formatted in the same way as and exactly matches the final answer in the provided final answer.\\
\verb+<|eot_id|><|start_header_id|>assistant<|end_header_id|>+
    \end{tracebox}
  \end{minipage}
  \caption[Prompt used to synthesize summaries and final answers into a detailed breakdown of reasoning steps.]{\textbf{Prompt used to synthesize summaries and final answers into a detailed breakdown of reasoning steps.}}
  \label{fig:app-expand-prompt-closed-stage1}
\end{figure}

\end{document}